%% file: main.tex
\documentclass[twoside,11pt]{article}

\usepackage[abbrvbib]{jmlr2e}

\usepackage[utf8]{inputenc}
\usepackage[T1]{fontenc}
\usepackage{lmodern}
\usepackage{microtype}
\usepackage{lastpage}

\usepackage{amsmath}
\usepackage{mathtools}
\usepackage{bm}

\PassOptionsToPackage{table}{xcolor}

\usepackage{placeins}
\usepackage{tikz}
\usetikzlibrary{arrows.meta, decorations.pathreplacing, shapes.multipart}
\usepackage{pgfplots}
\pgfplotsset{compat=1.18}

\usepackage{booktabs}
\usepackage{array}
\usepackage{multirow}
\usepackage{enumitem}
\usepackage[table]{xcolor}

\usepackage{algorithm}
\usepackage{algpseudocode}

\input{review.tex}

\hypersetup{colorlinks=true, linkcolor=blue, citecolor=blue, urlcolor=blue}
\makeatletter
\g@addto@macro\UrlBreaks{\do\/\do-\do_}
\makeatother
\usepackage[nameinlink,capitalize]{cleveref}
\crefname{figure}{Figure}{Figures}
\Crefname{figure}{Figure}{Figures}
\crefname{table}{Table}{Tables}
\Crefname{table}{Table}{Tables}

\renewenvironment{proof}[1][Proof]{\par\noindent{\bf #1\ }}{\hfill\BlackBox\\[2mm]}

\DeclareMathOperator{\Var}{Var}

\DeclareMathOperator{\diag}{diag}
\DeclareMathOperator{\softmax}{softmax}

\newcommand{\E}{\mathbb{E}}
\newcommand{\R}{\mathbb{R}}

\newcommand{\fl}{\mathrm{fl}}
\newcommand{\bfone}{\mathbf{1}}

\newcommand{\abs}[1]{\left\lvert #1\right\rvert}
\newcommand{\ulp}{\mathrm{ulp}}
\newcommand{\SR}{\mathrm{SR}}
\newcommand{\RN}{\mathrm{RN}}

\newcommand{\code}[1]{\texttt{#1}}
\newcommand{\SRbetter}[1]{\textcolor{green!45!black}{#1}}
\newcommand{\RNbetter}[1]{\textcolor{red!70!black}{#1}}
\newcommand{\safeincludegraphics}[2][]{%
  \IfFileExists{#2}{\includegraphics[#1]{#2}}{%
    \fbox{\parbox{0.9\linewidth}{\centering Missing figure: \path{#2}}}}}

\newcommand{\githubrepo}{\url{https://github.com/big-data-lab-team/fuzzy-llm}}

\jmlrheading{1}{2026}{1-\pageref{LastPage}}{09/26}{}{}{Yohan Chatelain and Pablo de Oliveira Castro}

\ShortHeadings{Stochastic Rounding in Transformer Inference}{Chatelain and de Oliveira Castro}
\firstpageno{1}

\begin{document}

\title{Stochastic Rounding in Low-Precision Transformer Inference:\\
  A Variable-Precision Emulation Study of a Small GPT-2}

\author{\name Yohan Chatelain\thanks{Equal contribution.} \email yohan.chatelain@camh.ca\thanks{Corresponding author.}\kern0.6em \\
  \addr Krembil Centre for Neuroinformatics\\
  Centre for Addiction and Mental Health\\
  250 College Street, Toronto, ON M5T 1R8, Canada
  \AND
  \name Pablo de Oliveira Castro\footnotemark[1] \email pablo.oliveira@uvsq.fr\kern0.6em \\
  \addr Universit\'e Paris-Saclay, UVSQ, LI-PaRAD\\
  Versailles, France}

\editor{}

\maketitle

\begin{abstract}%
  Should low-precision transformer inference use stochastic rounding~(SR) or
  round-to-nearest~(RN)? Holding the numerical format fixed and varying only the
  rounding rule at operation sites, we find that the answer is site-dependent.

  Two partial analyses account for this. A probabilistic forward- error bound for
  linear projections shows that SR's error envelope grows as
  $\mathcal{O}(\sqrt{n}\,u)$ in the reduction length $n$, against
  $\mathcal{O}(n\,u)$ for RN, a gap widest in the long multilayer perceptron (MLP)
  down-projection. A second-order decomposition of the expected cross- entropy
  change at the softmax separates signed drift, drift curvature, and a
  Fisher-weighted variance penalty: MLP noise is mostly a uniform logit shift, to
  which the softmax is invariant, whereas head noise is non-uniform.

  To test this,
  we extend PRISM, a vectorized CPU rounding library, with a variable-precision SR
  algorithm proven exact in hardware floating point. On DistilGPT-2 at \(t=6\)
  significand bits, SR in the MLP raises perplexity to $\times1.15$ the
  full-precision reference, against $\times2.21$ for RN. At the language-model
  head the ordering reverses, since only SR incurs the variance penalty. With the
  MLP output at $t=6$, assigning SR to the MLP and RN to the head brings
  perplexity within $\times1.10$ of the reference, $28\%$ below matched-bit RN.
\end{abstract}


\input{intro.tex}

\input{background.tex}

\input{algo.tex}
\input{prism.tex}

\input{analysis_1.tex}
\input{analysis_2.tex}

\input{experiments.tex}

\input{conclusion.tex}

\acks{This work was partially funded by the FPT-4 (ANR-24-CE46-7572) project.}

\appendix
\input{appendix.tex}

\input{appendix_drift.tex}

\bibliography{references}

\closereviewbar
\end{document}

%% file: review.tex
\newif\ifshowreviewmarks
\showreviewmarksfalse

\usepackage[xcolor,outerbars]{changebar}
\usepackage{amssymb}
\ifshowreviewmarks\else
  \nochangebars
\fi

\newif\ifreviewbaropen
\reviewbaropenfalse
\newif\ifreviewbarpaused
\reviewbarpausedfalse
\newcommand{\currentreviewbarcolor}{orange!70!black}
\newcommand{\closereviewbar}{%
  \ifreviewbaropen
    \cbend{}%
    \global\reviewbaropenfalse
  \fi}
\newcommand{\startreviewbar}[1]{%
  \closereviewbar
  \gdef\currentreviewbarcolor{#1}%
  \cbcolor{#1}\cbstart
  \global\reviewbaropentrue}
\newcommand{\pausereviewbar}{%
  \ifreviewbaropen
    \global\reviewbarpausedtrue
    \closereviewbar
  \else
    \global\reviewbarpausedfalse
  \fi}
\newcommand{\resumereviewbar}{%
  \ifreviewbarpaused
    \expandafter\startreviewbar\expandafter{\currentreviewbarcolor}%
    \global\reviewbarpausedfalse
  \fi}

\AddToHook{cmd/section/before}{\closereviewbar}
\AddToHook{cmd/subsection/before}{\closereviewbar}
\AddToHook{cmd/subsubsection/before}{\closereviewbar}
\AddToHook{env/algorithm/before}{\pausereviewbar}
\AddToHook{env/algorithm/after}{\resumereviewbar}
\AddToHook{env/table/before}{\pausereviewbar}
\AddToHook{env/table/after}{\resumereviewbar}
\AddToHook{env/table*/before}{\pausereviewbar}
\AddToHook{env/table*/after}{\resumereviewbar}
\AddToHook{env/figure/before}{\pausereviewbar}
\AddToHook{env/figure/after}{\resumereviewbar}
\AddToHook{env/figure*/before}{\pausereviewbar}
\AddToHook{env/figure*/after}{\resumereviewbar}

\newcommand{\reviewdetails}[1]{%
  \if\relax\detokenize{#1}\relax\else: #1\fi}

\makeatletter
\newcommand{\reviewer@one}[2]{%
  \ifshowreviewmarks\textcolor{blue}{[#1: #2]~}\fi}
\newcommand{\reviewer@resolved}[2]{%
  \ifshowreviewmarks
    \textcolor{green!45!black}{[\ensuremath{\checkmark}\ #1 resolved: #2]~}%
  \fi}
\newcommand{\reviewer@two}[3]{%
  \ifshowreviewmarks
    {\color{orange!80!black}#2} \textcolor{blue}{[#1: #3]~}%
  \else
    #2%
  \fi}
\newcommand{\newreviewer}[1]{%
  \expandafter\newcommand\csname #1\endcsname[1]{%
    \@ifnextchar\bgroup
    {\reviewer@two{#1}{##1}}%
    {\reviewer@one{#1}{##1}}%
  }%
  \expandafter\newcommand\csname #1resolved\endcsname[1]{%
    \reviewer@resolved{#1}{##1}}}
\makeatother

\newreviewer{YC}
\newreviewer{PO}

%% file: intro.tex
\section{Introduction}
\label{sec:intro}

Transformer inference increasingly uses formats narrower than the binary64
arithmetic for which much of classical numerical analysis was developed. The
dominant response is quantization: choose a target format, map weights and
activations onto it, and recover accuracy through calibration, outlier handling,
or per-channel scaling
\citep{3600270.3602468,xiao2023smoothquant,lin2024awq,gong2024what}. In that
program, the rounding operator is usually fixed by the underlying hardware format.
In post-training quantization, offline weight-rounding methods such as
AdaRound~\citep{nagel2020adaround} optimize the rounding direction of static
weights via a local second-order task loss surrogate. However, the rounding
rule used for intermediate calculations during inference has received less
attention. In contrast, dynamic rounding is already becoming explicit in
low-precision training: NVIDIA's microscaled NVFP4 recipe,
for example, selectively uses stochastic rounding for gradient casts to reduce
quantization bias~\citep{nvidia2025nvfp4}. Whether this same
selectivity is useful during inference remains an open question.

This paper isolates that question by holding the numerical format fixed and
varying only the rounding rule. Round-to-nearest~\citep{muller2018handbook} (RN)
returns the closer representable neighbor, whereas stochastic
rounding~\citep{forsythe1959rounding,croci2022stochastic} (SR), proposed by
Forsythe for controlling accumulated roundoff and pioneered in deep learning by
\citet{gupta2015deep} to prevent gradient stagnation in low-precision training,
selects either neighbor with probability proportional to proximity, making each
rounding error zero in expectation.
RN errors, by contrast, are deterministic functions of their operands and may
align; sufficiently small addends can be absorbed by a large accumulator or
repeated operations can drift away from the exact result. SR therefore trades
systematic local drift for higher variance.

Whether stochastic rounding improves inference accuracy depends first on the
operation. A long reduction is more exposed to accumulated RN drift than a short
one. In a GPT-2-style block~\citep{radford2019language}, the multilayer
perceptron (MLP) down-projection reduces over $d_{\rm ff}=3072$ terms, while an
attention-head score reduces over $d_{\rm head}=64$. At the same accumulator
precision, these operations can therefore have very different exposure to
accumulated roundoff.

The outcome also depends on the operation's location, both because local input
distributions affect how errors initially accumulate, and because downstream
layers transform the resulting perturbations. A perturbation in an attention score
passes through the softmax before changing the weighted value sum, whereas a
perturbation from the MLP down-projection enters the residual stream and is
transformed by later blocks. The same change of rounding rule can therefore
reduce perplexity at one site and increase it at another. This leads to the
research question of this paper: \emph{when and where does stochastic rounding
      improve transformer inference accuracy over round-to-nearest at low fixed
      precision?}

We study this question in the GPT-2 architecture, using DistilGPT-2 as a
six-block instance, and compare the two rounding rules site by site. Hardware
SR is tied to particular formats, operations, and accelerators, so experiments
at freely selected precisions require software emulation at every instrumented
operation, and the arithmetic intensity of transformer inference requires that
emulator to be fast and to vary precision and rounding rule at run time. The
contributions of this paper are as follows:

\begin{itemize}[leftmargin=1.5em]
      \item We present the \textbf{variable-precision stochastic rounding
                  (VPSR)} algorithm (\cref{sec:vpsr}), prove its correctness, and
            implement it in \textbf{PRISM}~\citep{prism}, where it runs up to
            $70\times$ faster than Verificarlo's Monte Carlo Arithmetic (MCA)
            random rounding (\cref{sec:prism-bench}).
      \item We analyze SR in transformer layers (\cref{sec:analysis}): we establish a
            vector Bienaym\'e--Chebyshev bound that extends prior scalar
            probabilistic forward-error analysis to the full matrix multiplication
            of the MLP projection. We also analyze softmax error propagation by
            deriving an exact decomposition of the cross-entropy loss change
            into signed target shift ($D$) and Kullback--Leibler (KL) distortion, and perform
            a second-order Taylor expansion of the KL term into drift curvature
            ($Q$) and variance penalty ($V$) terms, characterizing the
            site-level trade-offs between SR and RN.
      \item We evaluate on DistilGPT-2 and WikiText-2 (\cref{sec:experiments})
            by lowering $t$ site by site. SR is better in the MLP and RN at the
            language-model head, and the MLP gap widens as more down-projections
            are affected. At $t=6$ in the MLP, SR holds perplexity degradation to
            $15\%$ over the full-precision reference against $121\%$ for RN, and
            comes within $0.3\%$ of the reference by $t=8$.
\end{itemize}

\paragraph{Disclaimer on AI usage.}
AI tools (specifically Anthropic Claude, Google Gemini, and OpenAI Codex) were used during the preparation of this manuscript to assist with prose drafting, editing, and checking the mathematical proofs. All derivations, empirical evaluations, and final text were independently verified by the authors, who take full responsibility for the contents of this work.

%% file: background.tex
\section{Background and related work}
\label{sec:background}

\subsection{Stochastic rounding}

Let a binary floating-point format have $p$ significand bits, including the
implicit leading bit. We write $u = 2^{1-p}$
for its maximum relative ULP scale. Thus $u=2^{-23}$ for \code{binary32}
($p=24$) and $u=2^{-52}$ for \code{binary64} ($p=53$). This $u$ is the
spacing at $1$ and is twice the conventional RN unit roundoff. For a real $x$
that is not representable on $p$ bits, let $\lfloor x \rfloor$ be the
largest representable number below $x$ and $\lceil x \rceil$ the smallest above it, so
that $\lfloor x \rfloor < x < \lceil x \rceil$. Write $\ulp(x) = \lceil x \rceil - \lfloor x \rfloor$ for the
spacing of the floating-point grid in the binade containing $x$, which
satisfies~\citep{muller2018handbook}
\begin{equation}
  \frac{u}{2}\abs{x} < \ulp(x) \le u\abs{x}.
  \label{eq:ulp}
\end{equation}

Round-to-nearest returns whichever of $\lfloor x \rfloor$, $\lceil x \rceil$ is closer, resolving
ties by a fixed rule~\citep{ieee754-2019}. For non-representable $x$, stochastic rounding
returns:
\begin{equation*}
  \SR(x) =
  \begin{cases}
    \lceil x \rceil, & \text{with probability } p_x = \dfrac{x - \lfloor x \rfloor}{\lceil x \rceil - \lfloor x \rfloor}, \\[1.2ex]
    \lfloor x \rfloor, & \text{with probability } 1 - p_x,
  \end{cases}
\end{equation*}
so that $\E[\SR(x)] = x$. SR commits an error of magnitude at most
$\ulp(x)$; RN commits at most $\tfrac12\ulp(x)$. Consequently, for normal
results, RN has relative error at most $u/2$ and SR at most $u$.
RN therefore has the smaller worst-case error per operation, while
SR is unbiased on average.

Deterministic worst-case analysis of a length-$n$ dot product under RN yields
an error bound proportional to $nu$, and the bound is attained when the
individual errors align~\citep{higham2002accuracy}.
The mathematical properties of stochastic rounding are well established~\citep{croci2022stochastic}.
Under SR the errors form a martingale difference sequence, and probabilistic
backward error analysis replaces the $nu$ factor by one of order
$\sqrt{n}u$~\citep{higham2019new,connolly2021stochastic}. While this extends to
other multilinear kernels (e.g., polynomial evaluation and sum-product
trees~\citep{de2025error}), these bounds do not apply across nonlinear
operations. Under RN, an addend small enough relative to
the accumulator is absorbed entirely, so a sum of many small equal terms stalls
at a value set by $u$ rather than by the terms~\citep{higham2002accuracy}.
While stochastic rounding is sometimes criticized for the hardware overhead of
pseudorandom number generation, hardware implementations based on
linear-feedback shift registers make SR highly efficient in time,
energy, and area~\citep{benali2024stochastic}.

\subsection{Emulating stochastic rounding in software}
\label{sec:background-eft}

The standard construction uses \emph{error-free transforms} (EFTs). For an
elementary operation $x = a \circ b$, an EFT computes a pair $(\sigma,\tau)$ of
representable numbers with $\sigma = \fl(a \circ b)$, where $\fl(\cdot)$ denotes
floating-point rounding to nearest, and $\sigma + \tau = x$ exactly, recovering
the rounding residual.
\code{TwoSum}~\citep{dekker1971float} does this for addition in six operations,
and \code{TwoProdFMA}~\citep{graillat2006improving} does it for multiplication in
two when a fused multiply-add is available. Given $(\sigma,\tau)$, SR reduces to
comparing $\tau$ against a uniform threshold drawn on $[0, \ulp(\sigma))$.

Verrou~\citep{fevotte2016verrou} implemented this scheme, and Fasi and Mikaitis
subsequently formalized and analyzed it, giving pseudocode and correctness
arguments for all five elementary operations~\citep{fasi2021algorithms}. A second
mechanism is to use extended precision for the intermediate result and then to
round back to the working precision. This is implemented in Monte Carlo
Arithmetic~\citep{parker1997monte} with the \emph{random rounding} (MCA RR) mode,
available as a backend of Verificarlo~\citep{denis2016verificarlo}. Contrary to
the EFT route, MCA RR leaves the virtual precision $t$ free to be chosen by the
user, at the cost of evaluating each operation in higher precision.

In deep learning, simulation frameworks such as
QPyTorch~\citep{zhang2019qpytorch} support stochastic rounding across tensors,
but evaluate composite operations (such as matrix multiplications) using an
uninstrumented high-precision accumulator (\code{binary32}). Capturing error
accumulation or stagnation occurring along intermediate reduction chains
requires instruction-level emulation instead.

By contrast, compiler tools like Verificarlo~\citep{denis2016verificarlo}
transparently intercept elementary floating-point operations, allowing
PRISM~\citep{prism} to vectorize EFT rounding at the instruction level.
However, standard EFT methods round strictly at hardware precision ($t = p$);
\cref{sec:vpsr} resolves this limitation with the VPSR algorithm, decoupling
virtual precision from the hardware format without extended-precision overhead.

While hardware support exists, it is currently restricted to a few architectures
and formats/instructions. SR is not an IEEE-754 rounding attribute
\citep{ieee754-2019}, and commodity CPUs do not expose it. Graphcore IPUs couple
random generation to supported floating-point
arithmetic~\citep{graphcore2020aifloat}; AWS NeuronCore-v2 exposes an SR/RNE mode
for \code{FP32}-to-\code{BF16} autocasting~\citep{aws2023neuroncore}; and NVIDIA
PTX provides \code{cvt.rs} conversions with caller-supplied random bits on
supported Blackwell targets~\citep{nvidia2025ptx}. These mechanisms are tied to
particular formats, operations, and accelerators. Software emulation is
therefore a practical way to study SR on CPUs at a freely selected precision,
including precisions that have no hardware format.

\subsection{Low-precision transformer inference}

The quantization literature attacks a different part of the problem.
LLM.int8()~\citep{3600270.3602468} is an 8-bit integer quantization scheme for
transformer inference; it isolates the activation outliers that dominate
quantization error in large transformers into a higher-precision path.
SmoothQuant migrates activation difficulty into the
weights through a per-channel rescaling so that both become
quantizable~\citep{xiao2023smoothquant}. AWQ selects salient weight channels by
activation statistics and protects them~\citep{lin2024awq}. An empirical study by
\citet{gong2024what} frames all of these as perturbation control and asks which
perturbations matter.

All of these hold the rounding rule fixed and design the mapping of values onto
the grid. We do the opposite: the mapping is fixed and the rule is the
variable. The two address different error sources and can be applied together,
and the mechanism studied here, the interaction of rounding choice with
network propagation and downstream loss, is invisible to a formulation in which
rounding is a black box.

\paragraph{Stochastic rounding in the forward pass.}
The prevailing consensus confines SR to backward-pass gradient quantization and
rejects it for the forward pass. \citet{chmiel2023accurate} give
an explicit mathematical argument: because activation functions are nonlinear,
the unbiasedness SR provides at the tensor level
($\E[\SR(x)] = x$) does not survive composition with a nonlinear $f$, so that
$\E[f(\SR(x))] \neq f(x)$. They conclude that SR provides no benefit in the forward
pass and that RN, which minimises mean-squared error, should be preferred.
NVIDIA's NVFP4 pretraining recipe~\citep{nvidia2025nvfp4} follows the same
split, mandating SR only for gradient casts and using RN in the forward path.
In the forward pass, standard quantization recipes typically address
representation error by managing activation outliers, for example, using random
Hadamard transforms~\citep{ashkboos2024quarot,nvidia2025nvfp4}.  However,
outlier mitigation leaves the arithmetic reduction operator unaddressed: in
long inner products, deterministic RN errors still accumulate systematically
into directional drift.

Our analysis in \cref{sec:analysis} shows that this categorical rejection of SR
in the forward pass is too coarse: the performance of SR versus RN differs
across parts of the network.
In MLP projections, the long linear reduction allows deterministic RN errors
to accumulate into non-uniform drift, whereas SR limits this accumulation and
propagates predominantly as a uniform shift that the output softmax ignores
(\cref{sec:analysis-softmax}).
Conversely, at the output of the network (the language-model head), SR introduces
non-uniform fluctuations directly across classes, which cross-entropy loss
penalizes, shifting the advantage to RN.
A site-selective application of SR---using it at MLP projections while
retaining RN at pre-softmax sites---exploits these site-dependent dynamics
to control reduction drift directly at the arithmetic level.

\subsection{The architecture under study}

We use the sequential pre-LayerNorm GPT-2 block
\citep{radford2019language,vaswani2017attention}. Each layer first applies its
attention residual and then feeds the resulting stream to a separately
normalized MLP residual:
\begin{equation*}
  a_\ell = x_\ell + \mathrm{Attn}_\ell(\mathrm{LN}_{1,\ell}(x_\ell)),
  \qquad
  x_{\ell+1} = a_\ell + \mathrm{MLP}_\ell(\mathrm{LN}_{2,\ell}(a_\ell)),
\end{equation*}
where $x_\ell,a_\ell \in \R^{T \times d_{\rm model}}$ are residual streams at
layer $\ell$ for a sequence of $T$ tokens, $\mathrm{LN}_{1,\ell}$ and
$\mathrm{LN}_{2,\ell}$ are separate layer normalizations,
$\mathrm{Attn}_\ell$ is multi-head self-attention, and $\mathrm{MLP}_\ell$ is
the position-wise feedforward network. The normalization sits inside the
residual branch rather than after it~\citep{xiong2020layer}, so the stream
$x_\ell$ is never renormalized and a perturbation introduced at layer $\ell$
propagates additively to the output.

The MLP is a pair of affine maps around a GELU
nonlinearity $\phi$~\citep{hendrycks2016gelu},
\begin{equation}
  z = \widetilde a W_1 + b_1, \qquad h = \phi(z), \qquad y = h W_2 + b_2,
  \label{eq:mlp}
\end{equation}
with $W_1 \in \R^{d_{\rm model} \times d_{\rm ff}}$ and
$W_2 \in \R^{d_{\rm ff} \times d_{\rm model}}$. Attention first forms queries,
keys and values by a single projection of the normalized stream,
$[Q\,|\,K\,|\,V] = \tilde{x} W_{\rm qkv} + b_{\rm qkv}$ (the \code{attn\_c\_attn}
matrix multiplication), splits the result across $H$ heads, and for each head
computes
\begin{equation*}
  S = \frac{QK^\top}{\sqrt{d_{\rm head}}}, \qquad
  A = \softmax(S), \qquad
  O = AV,
\end{equation*}
where the softmax is taken row-wise under a causal mask. The per-head outputs
are concatenated and projected back to the stream by $W_{\rm o}$ (the
\code{attn\_c\_proj} matrix multiplication). After the last block, a final linear map
$W_{\rm head} \in \R^{d_{\rm model} \times V}$ produces logits over a
vocabulary of $V$ tokens, which are scored by cross-entropy against the next
token.

The dimensions matter for what follows, because they set reduction lengths. The lower panel of \cref{fig:architecture} places the reductions of one block on the dataflow. Three of the four matrix
multiplications reduce over $d_{\rm model}$ (\code{attn\_c\_attn},
\code{attn\_c\_proj}, \code{mlp\_c\_fc}) and the fourth over
$d_{\rm ff}$ (\code{mlp\_c\_proj}), while inside attention the score
product reduces over $d_{\rm head}$ and the value product over the
sequence length. The specific dimensions of the evaluated model are detailed in \cref{sec:experiments}.

\Cref{fig:architecture} places the block-level dataflow inside the complete DistilGPT-2 evaluation path, from token and
position embeddings through the six-block stack to the final normalization,
language-model head, cross-entropy, and perplexity.

\input{gpt2_full}

%% file: gpt2_full.tex
\usetikzlibrary{arrows.meta,calc,positioning}
\definecolor{gptpanel}{RGB}{237,243,248}
\definecolor{gptattention}{RGB}{252,229,207}
\definecolor{gptmlp}{RGB}{226,239,204}
\definecolor{gptreduction}{RGB}{255,244,204}
\definecolor{gptresidual}{RGB}{76,101,125}

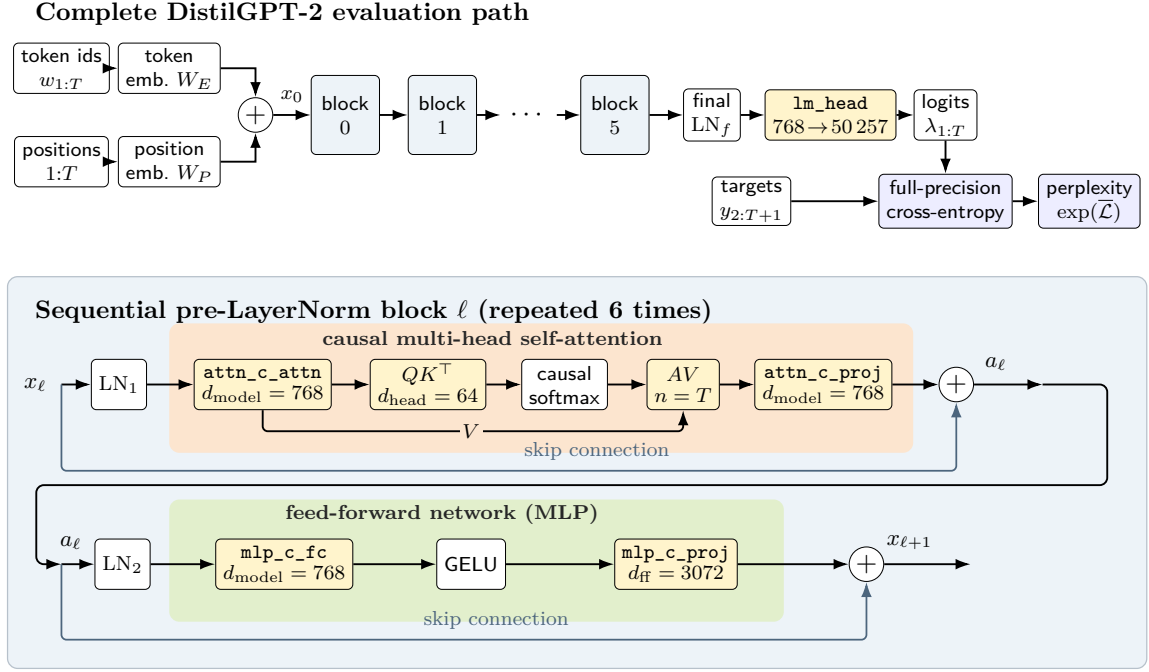
\begin{figure}[t]
  \centering
  \resizebox{0.99\linewidth}{!}{%
    \begin{tikzpicture}[
      x=1cm,
      y=1cm,
      font=\sffamily\footnotesize,
      >={Latex[length=2.1mm,width=1.5mm]},
      flow/.style={->,line width=0.75pt},
      residual/.style={->,line width=0.7pt,draw=gptresidual},
      tie/.style={<->,densely dashed,line width=0.65pt,draw=gptresidual},
      operation/.style={
          draw,
          rounded corners=2pt,
          fill=white,
          minimum height=7mm,
          inner xsep=3pt,
          inner ysep=1.5pt,
          align=center,
          font=\sffamily\scriptsize
        },
      reduction/.style={operation,fill=gptreduction},
      block/.style={operation,fill=gptpanel,minimum width=9.5mm,minimum height=11mm},
      metric/.style={operation,fill=blue!7,minimum height=7.5mm},
      add/.style={
          circle,
          draw,
          fill=white,
          minimum size=4.8mm,
          inner sep=0pt,
          font=\normalsize
        }
      ]

      \node[anchor=north west,font=\bfseries\small]
      at (0.05, 8.05) {Complete DistilGPT-2 evaluation path};

      \node[operation] (tokens) at (0.55, 7.00) {token ids\\$w_{1:T}$};
      \node[operation] (positions) at (0.55, 5.70) {positions\\$1{:}T$};
      \node[operation] (wte) at (2.05, 7.00) {token\\emb.\ $W_E$};
      \node[operation] (wpe) at (2.05, 5.70) {position\\emb.\ $W_P$};
      \node[add] (embadd) at (3.25, 6.35) {$+$};

      \node[block] (bzero) at (4.50, 6.35) {block\\$0$};
      \node[block] (bone) at (5.85, 6.35) {block\\$1$};
      \node[draw=none,font=\bfseries] (dots) at (7.05, 6.35) {$\cdots$};
      \node[block] (bfive) at (8.25, 6.35) {block\\$5$};

      \node[operation] (lnf) at (9.60, 6.35) {final\\$\mathrm{LN}_f$};
      \node[reduction] (head) at (11.25, 6.35) {\code{lm\_head}\\$768\!\to\!50\,257$};
      \node[operation] (logits) at (12.85, 6.35) {logits\\$\lambda_{1:T}$};

      \node[operation,minimum height=5.5mm] (targets) at (10.15, 5.15) {targets\\$y_{2:T+1}$};
      \node[metric] (loss) at (12.85, 5.15) {full-precision\\cross-entropy};
      \node[metric] (ppl) at (14.85, 5.15) {perplexity\\$\exp(\overline{\mathcal L})$};

      \draw[flow] (tokens) -- (wte);
      \draw[flow] (positions) -- (wpe);
      \draw[flow] (wte.east) -| (embadd.north);
      \draw[flow] (wpe.east) -| (embadd.south);

      \draw[flow] (embadd) -- node[above=1pt,font=\scriptsize] {$x_0$} (bzero);
      \draw[flow] (bzero) -- (bone);
      \draw[flow] (bone) -- (dots);
      \draw[flow] (dots) -- (bfive);
      \draw[flow] (bfive) -- (lnf);
      \draw[flow] (lnf) -- (head);
      \draw[flow] (head) -- (logits);

      \draw[flow] (logits) -- (loss);
      \draw[flow] (targets) -- (loss);
      \draw[flow] (loss) -- (ppl);

      \fill[gptpanel,rounded corners=4pt,draw=gptresidual!40,line width=0.5pt]
      (-0.20, -1.35) rectangle (15.65, 4.10);

      \node[anchor=north west,font=\bfseries\small]
      at (0.05, 3.93) {Sequential pre-LayerNorm block $\ell$ (repeated 6 times)};

      \fill[gptattention,rounded corners=4pt]
      (2.05, 1.65) rectangle (12.40, 3.45);
      \node[font=\bfseries\scriptsize,text=black!85]
      at (6.55, 3.25) {causal multi-head self-attention};

      \coordinate (xin) at (0.55, 2.60);
      \node[anchor=east] at (0.48, 2.60) {$x_\ell$};
      \node[operation] (lnone) at (1.35, 2.60) {$\mathrm{LN}_1$};
      \node[reduction] (cattn) at (3.35, 2.60)
      {\code{attn\_c\_attn}\\[-0.45mm]$d_{\rm model}=768$};
      \node[reduction] (qk) at (5.65, 2.60)
      {$QK^\top$\\[-0.45mm]$d_{\rm head}=64$};
      \node[operation] (softmax) at (7.55, 2.60)
      {causal\\softmax};
      \node[reduction] (av) at (9.20, 2.60)
      {$AV$\\[-0.45mm]$n=T$};
      \node[reduction] (attnproj) at (11.15, 2.60)
      {\code{attn\_c\_proj}\\[-0.45mm]$d_{\rm model}=768$};
      \node[add] (addone) at (13.00, 2.60) {$+$};
      \coordinate (amid) at (14.20, 2.60);
      \node[anchor=south] at (13.55, 2.67) {$a_\ell$};

      \draw[flow] (xin) -- (lnone.west);
      \draw[flow] (lnone.east) -- (cattn.west);
      \draw[flow] (cattn.east) -- (qk.west);
      \draw[flow] (qk.east) -- (softmax.west);
      \draw[flow] (softmax.east) -- (av.west);
      \draw[flow,rounded corners=2pt]
      (cattn.south) -- (3.35, 1.95) -- (9.20, 1.95) -- (av.south);
      \node[font=\scriptsize,fill=gptattention,inner sep=1pt]
      at (6.25, 1.95) {$V$};
      \draw[flow] (av.east) -- (attnproj.west);
      \draw[flow] (attnproj.east) -- (addone.west);
      \draw[flow] (addone.east) -- (amid);

      \draw[residual,rounded corners=2pt]
      (xin) -- (0.55, 1.40) -- (13.00, 1.40) -- (addone.south);
      \node[anchor=south,draw=none,font=\scriptsize,text=gptresidual]
      at (8.00, 1.42) {skip connection};

      \fill[gptmlp,rounded corners=4pt]
      (2.05, -0.70) rectangle (10.60, 1.00);
      \node[font=\bfseries\scriptsize,text=black!85]
      at (5.85, 0.80) {feed-forward network (MLP)};

      \coordinate (ain) at (0.55, 0.10);
      \node[anchor=south] at (0.68, 0.18) {$a_\ell$};
      \node[operation] (lntwo) at (1.40, 0.10) {$\mathrm{LN}_2$};
      \node[reduction] (fc) at (3.65, 0.10)
      {\code{mlp\_c\_fc}\\[-0.45mm]$d_{\rm model}=768$};
      \node[operation,fill=white] (gelu) at (6.25, 0.10) {GELU};
      \node[reduction] (mlpproj) at (9.10, 0.10)
      {\code{mlp\_c\_proj}\\[-0.45mm]$d_{\rm ff}=3072$};
      \node[add] (addtwo) at (11.75, 0.10) {$+$};
      \coordinate (xout) at (13.20, 0.10);
      \node[anchor=south] at (12.35, 0.17) {$x_{\ell+1}$};

      \draw[flow] (ain) -- (lntwo.west);
      \draw[flow] (lntwo.east) -- (fc.west);
      \draw[flow] (fc.east) -- (gelu.west);
      \draw[flow] (gelu.east) -- (mlpproj.west);
      \draw[flow] (mlpproj.east) -- (addtwo.west);
      \draw[flow] (addtwo.east) -- (xout);

      \draw[residual,rounded corners=2pt]
      (ain) -- (0.55, -0.95) -- (11.75, -0.95) -- (addtwo.south);
      \node[anchor=south,draw=none,font=\scriptsize,text=gptresidual]
      at (6.60, -0.93) {skip connection};

      \draw[flow,rounded corners=3pt]
      (amid) -- (15.10, 2.60) -- (15.10, 1.20) --
      (0.20, 1.20) -- (0.20, 0.10) -- (ain);

        \end{tikzpicture}%
  }
  \caption{Complete DistilGPT-2 architecture and evaluation path. Token and
    position embeddings form $x_0$, which passes sequentially through six
    pre-LayerNorm GPT-2 blocks. The lower panel expands the
    repeated block into its attention and MLP sublayers. Yellow nodes identify the groups
    instrumented by the rounding experiments.}
  \label{fig:architecture}
\end{figure}

%% file: algo.tex
\begin{figure}[t]
    \centering
    \begin{tikzpicture}[scale=1.1, every node/.style={scale=0.95}]
        \draw[thick, -{Stealth}] (-0.5,0) -- (6.5,0);

        \foreach \x in {1.0, 2.0, 3.0, 4.0, 5.0} {
            \draw[gray!70, thick] (\x, -0.08) -- (\x, 0.08);
        }
        \draw[decorate, decoration={brace, amplitude=3pt}, gray!80]
            (1.0, 0.15) -- (2.0, 0.15) node[midway, above=2pt, font=\tiny] {$\ulp_{\rm hw}$};

        \filldraw (0,0) circle (2pt) node[below=8pt] {$\tilde{\sigma}$};
        \filldraw (6,0) circle (2pt) node[below=8pt] {$\tilde{\sigma} + \ulp_t$};

        \coordinate (X) at (5.3, 0);
        \filldraw[red] (X) circle (2pt) node[above=2pt, red] {\textbf{Exact $x$}};

        \draw[blue, thick] (0, 0.6) -- (5.0, 0.6) node[midway, above] {$\rho$};

        \draw[orange, thick] (5.0, 0.6) -- (5.3, 0.6) node[midway, above=1pt] {\footnotesize $\tau$};

        \fill[red, opacity=0.1] (0,-0.2) rectangle (5.3, 0.9);

        \node (pi) at (3.5, -1) [draw, circle, inner sep=2pt, fill=gray!10] {$\pi$};
        \draw[-{Latex}, gray, dashed, thick] (pi) -- (3.5, -0.1);

        \draw[decorate, decoration={brace, mirror, amplitude=6pt}] (0,-1.3) -- (6,-1.3)
            node[midway, below=8pt] {Sample threshold $\pi$ on $[0,\ulp_t)$};
    \end{tikzpicture}
    \caption{Geometric intuition for variable-precision stochastic rounding (VPSR). An exact value $x$ lies between virtual grid points at precision $t$. The hardware rounding $\sigma$ is truncated to $\tilde{\sigma}$, leaving a virtual truncation error $\rho$ and hardware residual $\tau$. A threshold $\pi$ is sampled on $[0,\ulp_t)$ to determine the virtual rounding direction.}
    \label{fig:vpsr_eft}
\end{figure}
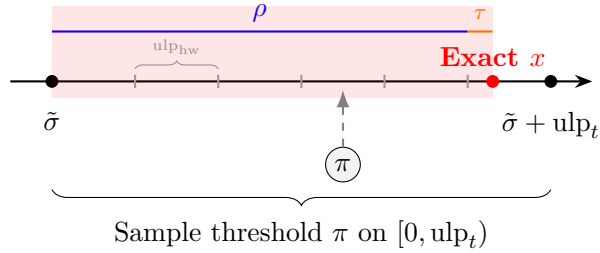

\begin{algorithm}[t]
\caption{Variable-precision stochastic rounding}
\label{alg:vpsr}
\begin{algorithmic}[1]
\Require Exact inputs $\sigma, \tau$, virtual precision $t$
\Ensure Rounded value $r$ at virtual precision $t$
\State $\tilde{\sigma} \leftarrow \mathrm{trunc}_t(\sigma)$ \Comment{Truncated value at precision $t$}
\State $\rho \leftarrow \sigma - \tilde{\sigma}$ \Comment{Exact virtual truncation error}
\State \textbf{if} $\rho = 0$ \textbf{ and } $\tau = 0$ \textbf{then} \Return $\tilde{\sigma}$ \Comment{Flattened to a mask in SIMD}
\State $s_{\rm dir} \leftarrow \begin{cases} -1 & \text{if } (\rho = 0 \text{ and } \tau < 0) \text{ or } \rho < 0 \\ 1 & \text{otherwise} \end{cases}$ \Comment{Determine rounding direction sign}
\Statex \Comment{The choice of $\eta$ ensures $\ulp_t$ will be the ULP of the exact value $x$ at precision $t$, i.e., $\ulp_t(x)$}
\State $\eta \leftarrow \begin{cases} \mathrm{exponent}(\mathrm{predecessor\_abs}(\tilde{\sigma})) & \text{if } (s_{\rm dir} < 0) \neq (\tilde{\sigma} < 0) \\ \mathrm{exponent}(\tilde{\sigma}) & \text{otherwise} \end{cases}$
\State $\ulp_t \leftarrow s_{\rm dir} \cdot 2^{\eta - (t - 1)}$ \Comment{Compute signed ULP at virtual precision $t$ using exponent $\eta$}
\State $S \leftarrow \begin{cases} 2^{64} & \text{if } \eta - (t - 1) \le e_{\min} \\ 1 & \text{otherwise} \end{cases}$ \Comment{Scale subnormals to normal range}
\State $\rho_S \leftarrow \rho \cdot S, \quad \tau_S \leftarrow \tau \cdot S, \quad \ulp_{t,S} \leftarrow \ulp_t \cdot S$
\State Sample $z$ uniformly from $\{k2^{-r}:0\le k<2^r\}$
\State $\pi \leftarrow \ulp_{t,S} \cdot z$ \Comment{Random threshold}
\State $D \leftarrow (\rho_S - \pi) + \tau_S$ \Comment{Decision boundary term}
\State $\Delta \leftarrow \begin{cases} \ulp_t & \text{if } D \cdot s_{\rm dir} \ge 0 \\ 0 & \text{otherwise} \end{cases}$
\State \Return $\tilde{\sigma} + \Delta$
\end{algorithmic}
\end{algorithm}

\section{Variable-precision stochastic rounding (VPSR)}
\label{sec:vpsr}

This section presents the variable-precision stochastic rounding (VPSR) algorithm, which decouples the rounding decision from the hardware format by decomposing each exact arithmetic result onto a coarser virtual grid. We prove that the rounding decision is evaluated exactly in hardware floating point, with no extended-precision overhead.

At hardware precision, the EFT of \cref{sec:background-eft} recovers the exact result $x = \sigma + \tau$ with $\sigma = \fl(x)$, and the stochastic rounding algorithm formalized by \citet{fasi2021algorithms} compares the hardware residual $\tau$ against a uniform threshold on $[0, \ulp_{\rm hw}(\sigma))$.

Variable-precision stochastic rounding (VPSR) extends this algorithmic framework to a coarser virtual grid at a precision $t \le p$ instead, where $p$ is the hardware precision. Truncating the
head to that grid, $\tilde{\sigma} = \mathrm{trunc}_t(\sigma)$, splits the
distance from $x$ down to the virtual grid point into a virtual truncation error
$\rho = \sigma - \tilde{\sigma}$ and the hardware residual $\tau$
(\cref{fig:vpsr_eft}),
\begin{equation*}
  x = \tilde{\sigma} + \rho + \tau .
\end{equation*}

To perform ideal stochastic rounding, \cref{alg:vpsr} draws a continuous threshold
$\pi \sim \mathcal{U}[0, \ulp_t)$ and rounds away from $\tilde\sigma$ when the
signed threshold does not exceed the signed distance $\rho+\tau$. Fixing
$\pi = \tfrac12\ulp_t$ instead gives untied round-to-nearest at
precision $t$.

Because hardware generators provide discrete bits rather than continuous draws, PRISM approximates the uniform distribution using $r$ random bits ($r=24$ for \code{binary32} and $r=53$ for \code{binary64}). Let $\SR_{t,r}$ denote this finite-bit implementation. This discretization introduces a negligible bias bounded by the number of random bits $r$:
\begin{equation*}
  \abs{\E[\SR_{t,r}(x)]-x}\le 2^{-r}\abs{\ulp_t(x)}.
\end{equation*}

Three steps that are trivial at hardware precision require care at virtual precision:
\begin{enumerate}[leftmargin=1.5em]
  \item \textbf{Two-term decision.} The relevant distance from $x$ to the virtual grid point is $\rho + \tau$, not $\tau$ alone. These terms can differ by the full hardware precision, making the sign of their floating-point sum fragile. \Cref{sec:proof} proves that evaluating $(\rho - \pi) + \tau$ in hardware arithmetic nevertheless preserves the exact sign, requiring no compensated arithmetic.
  \item \textbf{Power-of-two boundaries.} When the rounding direction carries $x$ across a binade boundary, line~5 of \cref{alg:vpsr} selects the predecessor exponent so that $\ulp_t$ reflects the grid spacing at $x$ rather than at $\tilde{\sigma}$.
  \item \textbf{Subnormal virtual grid.} When $\ulp_t$ is subnormal, the threshold product $\ulp_t \cdot z$ loses precision. Lines~7--8 scale all quantities by $S = 2^{64}$ (exact as a power of two) and compare in the normal range; the decision is invariant under common positive scaling.
\end{enumerate}

\subsection{Exactness of the rounding decision}
\label{sec:proof}

We assume IEEE~754 floating-point arithmetic with round-to-nearest. Let
$\fl(x+y)$ and $\fl(x-y)$ denote floating-point addition and subtraction,
respectively.

We define the rounding direction sign $s_{\rm dir}$ (equivalently to line~4 of \cref{alg:vpsr}, noting that the $\rho = 0, \tau = 0$ case is handled on line~3):
\[
s_{\rm dir} = 
\begin{cases} 
1 & \text{if } \rho > 0 \text{ or } (\rho = 0 \text{ and } \tau > 0) \\
-1 & \text{if } \rho < 0 \text{ or } (\rho = 0 \text{ and } \tau < 0)
\end{cases}
\]

Throughout the theorem and its proof, $\ulp_{\rm hw}$ abbreviates
$\ulp_{\rm hw}(x)$, the hardware grid spacing at $x$. This is an absolute
spacing, not the relative ULP scale $u = 2^{1-t}$ of \cref{eq:ulp}.

\begin{theorem}[Exact sign evaluation of the rounding decision]
\label{thm:exact-sign}
Let $\rho$, $\pi$, and $\tau$ be floating-point numbers. If
the following properties hold:
\begin{enumerate}
    \item $|\tau| \le \ulp_{\rm hw}/2$
    \item $\rho = 0$ or $|\rho| \ge \ulp_{\rm hw}$
    \item $\rho$ and $\pi$ share the same sign
\end{enumerate}
Then the floating-point evaluation of the decision boundary preserves the sign:
\[
(\rho - \pi + \tau) \cdot s_{\rm dir} \ge 0 \iff \fl(\fl(\rho - \pi) + \tau) \cdot s_{\rm dir} \ge 0
\]
\end{theorem}

The three hypotheses hold in \cref{alg:vpsr}: (1)~is the standard RN bound $|\tau| \le \ulp_{\rm hw}/2$; (2)~holds because truncation zeros trailing bits without changing the exponent, making $\rho$ an integer multiple of $\ulp_{\rm hw}(\sigma)$; since IEEE~754 monotonicity gives $|x| \ge 2^k \Rightarrow |\fl(x)| \ge 2^k$, the rounded value $\sigma$ cannot cross a binade boundary toward zero, so $\ulp_{\rm hw}(\sigma) \ge \ulp_{\rm hw}$; (3)~follows from $|\rho| \ge \ulp_{\rm hw} \ge 2|\tau|$, which prevents $\tau$ from flipping the sign of $\rho$, combined with the algorithm assigning $\pi$ the sign of $s_{\rm dir}$. The subnormal scaling on lines~7--8 preserves all three by multiplying distances and spacing by a common exact power of two.

\begin{proof}[Proof sketch]
We normalize the variables by the rounding direction sign $s_{\rm dir} \in \{-1, 1\}$, setting $\rho' = \rho \cdot s_{\rm dir}$, $\pi' = \pi \cdot s_{\rm dir}$, and $\tau' = \tau \cdot s_{\rm dir}$, so that the equivalence becomes $(\rho' - \pi') + \tau' \ge 0 \iff \fl(\fl(\rho' - \pi') + \tau') \ge 0$.

When $\rho = 0$, the first subtraction $\fl(0 - \pi') = -\pi'$ is exact in IEEE~754 arithmetic. The decision then reduces to a single floating-point addition $\fl(-\pi' + \tau')$, whose sign matches the exact sign by monotonicity.

When $\rho \ne 0$, monotonicity directly ensures the forward implication. For the backward implication, the only scenario where the signs could disagree is an artificial collapse to zero, $\fl(\fl(\rho' - \pi') + \tau') = 0$, while the exact sum is strictly negative, $(\rho' - \pi') + \tau' < 0$. Under the bounds from Hypotheses~1 and~2, such a collapse requires the intermediate difference $\fl(\rho' - \pi')$ to equal $-\tau'$, forcing $\rho'/2 \le \pi' \le 2\rho'$. This condition satisfies the hypotheses of Sterbenz's Lemma~\citep{sterbenz1974floating}, which guarantees that the floating-point subtraction $\fl(\rho' - \pi')$ is exact. Exact subtraction then yields $(\rho' - \pi') + \tau' = 0$, contradicting the assumption that the exact sum is strictly negative. Thus, no sign discrepancy can occur. The complete formal proof is deferred to \cref{app:proof-exact-sign}.
\end{proof}

%% file: prism.tex
\section{PRISM: vectorized probabilistic rounding at variable precision}
\label{sec:prism}

This section describes how \cref{alg:vpsr} is applied to every floating-point
operation of a PyTorch model and what it costs. We implement VPSR in PRISM
(\emph{Probabilistic Rounding with Instruction Set Management}), an open-source
C++17 library that previously applied vectorized SR only at hardware
precision~\citep{gonzalezpepe2026fuzzy}. The Verificarlo compiler
pass~\citep{denis2016verificarlo} replaces each floating-point instruction of
the compiled framework with a call to a PRISM kernel, which computes the exact
result as an EFT pair and rounds it at a virtual precision $t$ taken from a
run-time configuration (\cref{fig:prism-pipeline}). In compiler-instrumented
loops, PRISM SR costs $8.5$--$32.1\times$ native time, $4.4$--$70.5\times$ less
than Verificarlo's MCA random rounding (\cref{sec:prism-bench}).

\begin{figure}[htb]
  \centering
  \begin{tikzpicture}[
    font=\small,
    hbox/.style={draw, rounded corners=2pt, align=center, inner sep=3pt,
        rectangle split, rectangle split parts=2,
        rectangle split part fill={gray!25, gray!6}},
    arr/.style={-{Stealth}, thick}]
    \node[hbox, text width=2.2cm] (op) at (0,0)
    {\textbf{PyTorch}
      \nodepart{two} \code{c = a * b}};
    \node[hbox, text width=2.9cm] (entry) at (4.16,0)
    {\textbf{PRISM call}
      \nodepart{two} \code{mulf32x16(a,b)}};
    \node[hbox, text width=2.8cm] (kernel) at (8.43,0)
    {\textbf{VPSR kernel}
      \nodepart{two} EFT $(\sigma,\tau)$,\\ then \cref{alg:vpsr}};
    \node[hbox, text width=2.2cm] (out) at (11.95,0)
    {\textbf{Result}
      \nodepart{two} $c$ rounded at\\ precision $t$};
    \node[hbox] (cfg) at (8.43,1.7)
    {\textbf{Configuration}
      \nodepart{two} virtual precision $t$};
    \draw[arr] (op) -- node[above, font=\scriptsize, align=center]
    {Verificarlo\\ pass} (entry);
    \draw[arr] (entry) -- node[below, font=\scriptsize] {16 lanes} (kernel);
    \draw[arr] (kernel) -- (out);
    \draw[arr] (cfg) -- (kernel);
  \end{tikzpicture}
  \caption{Execution of one instrumented operation. At compile time, the
    Verificarlo pass replaces a 512-bit \code{binary32} multiplication with a
    call to the PRISM entry point for $16$ lanes. At run time, the kernel
    rounds all lanes with \cref{alg:vpsr} at the precision set.}
  \label{fig:prism-pipeline}
\end{figure}
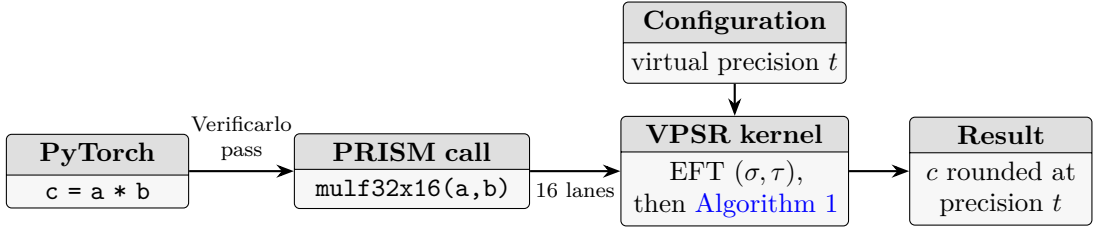

\subsection{Instrumentation and kernels}
\label{sec:prism-kernels}

The Verificarlo pass maps each scalar or vector floating-point instruction to a
PRISM entry point for the same operation, format and lane count, so PyTorch's
source is unchanged. Each call rounds all its lanes in one kernel invocation.
We use PRISM's static build, which fixes the instruction set at compile time.

Each kernel computes the EFT pair $(\sigma,\tau)$ of \cref{sec:background-eft}
and then applies \cref{alg:vpsr}. SR and RN share the kernel: SR draws $z$ from
$r$ random bits, and RN fixes $z=1/2$. Because the final comparison of
\cref{alg:vpsr} is non-strict, an exact midpoint rounds away from
$\tilde\sigma$, and hence away from zero; giving \emph{untied} RN.
FMA is rounded as one operation, using the \code{ErrFmaNearest}
transform~\citep{boldo2011exact} to represent $ab+c$ exactly.

The kernels are branch-free. Every data-dependent condition, in
\cref{alg:vpsr} and in the EFTs, is computed as a lane mask and resolved by
selection, so all lanes follow one instruction stream whatever their values,
including infinities and NaNs. The kernels are written once against Google
Highway~\citep{highway}, a C++ SIMD abstraction, and compile to its x86, Arm,
RISC-V and WebAssembly targets; where no hardware FMA exists, \code{TwoProdFMA}
uses the emulated FMA of \citet{graillat2023emulation}.

\subsection{Run-time configuration}
\label{sec:prism-config}

The virtual precisions for \code{binary32} and \code{binary64}, the rounding
mode and the generator state are run-time settings, so changing them requires
no recompilation. In our experiments, PyTorch module hooks change them before
and after each target module (\cref{sec:exp-scoping}). Kernels read the
settings from a per-thread cache. A process-wide setter publishes new settings
and advances a configuration epoch, and each worker thread reloads its cache
when it observes a new epoch on its next operation, so a change made from the
Python thread reaches all of PyTorch's worker threads.

\subsection{Performance}
\label{sec:prism-bench}

\paragraph{Cost model.}
\label{sec:prism-cost}
Emulation replaces one instruction with an EFT (two to six operations), about
twenty-five vector operations for the rounding decision, and part of a random
draw. The control flow is fixed, so this cost does not depend on the operands.
The slowdown relative to native execution is therefore smaller when native
execution is memory-bound and larger when a call covers few lanes.

\paragraph{Benchmark.}
We time elementwise addition, multiplication, division and FMA over arrays of
$65{,}536$ values in \code{binary32} and \code{binary64}. Each loop is compiled
with \code{-O3} for 512-bit AVX-512 vectors in three builds: without
instrumentation (the native baseline), with the PRISM SR pass, and with
Verificarlo's MCA QUAD backend in random-rounding (RR) mode
(\cref{sec:background-eft}). We repeat the measurement $30$ times on Intel Xeon
Gold 6130 @ 2.1~GHz CPU (cache sizes: 32~KB L1, 1~MB L2, 22~MB L3 shared).

\begin{table}[t]
  \centering
  \caption{Slowdown of PRISM SR versus MCA RR relative to native execution
  ($N=65{,}536$, 30 runs; $t=24$ for \code{binary32}, $t=53$ for
  \code{binary64}). PRISM and MCA entries are the median ratios to
  native time for the same operation and data type (lower is better); MCA/PRISM
  is the median per-run ratio of MCA to PRISM time.}
  \label{tab:prism-vector-perf}
  \begin{tabular}{lcccccc}
    \toprule
                                                   & \multicolumn{3}{c}{\code{binary32}} &
    \multicolumn{3}{c}{\code{binary64}}                                                                                                                    \\
    \cmidrule(lr){2-4}\cmidrule(lr){5-7} Operation & PRISM                               & MCA
                                                   & MCA/PRISM                           & PRISM             & MCA   & MCA/PRISM                           \\
    \midrule
    add                                            & $31.1$                              & $142.4$           & $4.7$ & $11.6$           & $115.6$ & $10.0$ \\
    \rowcolor{black!6} mul                         & $32.1$                              & $143.7$           & $4.5$ & $11.3$           & $163.1$ & $14.6$ \\
    div                                            & $25.6$                              & $112.7$           & $4.4$ & $12.0$           & $229.9$ & $19.2$ \\
    \rowcolor{black!6} fma                         & $19.1$                              & $\phantom{0}88.3$ & $4.7$ & $\phantom{0}8.5$ & $597.4$ & $70.5$ \\
    \bottomrule
  \end{tabular}
\end{table}

\paragraph{Results.}
PRISM SR incurs an overhead of $\times8.5$--$32.1$, whereas MCA RR incurs
$\times88$--$597$ (\cref{tab:prism-vector-perf}). Compared to MCA RR, PRISM SR
is up to $\times4.7$ faster than for \code{binary32} operations and up to
$\times70.5$ faster for \code{binary64}. The gap comes from how each backend
obtains the exact result it rounds. MCA RR computes the result in a wider format
and adds its random perturbation there: in \code{binary64} for \code{binary32}
operations, which the hardware supports, and in \code{binary128} for
\code{binary64} operations, which most CPUs emulate in software. PRISM recovers
the exact result with EFTs in the working format, so both formats use hardware
instructions only. Performance is nearly independent of the virtual precision:
across precisions, PRISM times vary by less than $1\%$ and MCA times by at most
$5\%$.

%% file: analysis_1.tex
\section{Theoretical error analysis of SR versus RN}
\label{sec:analysis}

The GPT-2 architecture contains projections with very different reduction
lengths---the MLP down-projection accumulates over $n=3072$ terms, whereas the
attention projections accumulate over $n=768$---and the language-model head
applies the softmax directly to the perturbed logits.  These structural
differences create distinct opportunities for SR's variance--drift trade-off.
This section analyzes both regimes:
first, establishing a probabilistic forward error bound for the affine
projection layers (\code{attn\_c\_proj} and \code{mlp\_c\_proj}) via a vector Bienaym\'e--Chebyshev
construction (\cref{sec:analysis-mlp}); second, deriving an analytical decomposition
of cross-entropy loss change into drift and variance components at the output softmax
(\cref{sec:analysis-softmax}); and third, empirically validating this trade-off on DistilGPT-2
(\cref{sec:analysis-empirical-drift}).
Throughout, statements concern the precision at which a reduction \emph{accumulates}.
On physical hardware, a kernel nominally in \code{bfloat16} may store in \code{bfloat16},
multiply into a wider product, and accumulate in \code{binary32}; under PRISM,
every elementary operation rounds at the same virtual precision, $u_{\rm sto} = u_{\rm mul} = u_{\rm acc} = 2^{1-t}$,
so the relative ULP scale $u$ below is $2^{1-t}$ for the significand precision $t$
swept by our harness.

\subsection{Forward error analysis of linear projection layers}
\label{sec:analysis-mlp}

The multilayer perceptron (MLP) block in the GPT-2 architecture consists of two
consecutive linear projections surrounding an element-wise activation function.
Following the row-vector convention of \cref{eq:mlp},
\begin{equation*}
  \text{MLP}(x) = \phi(x W_1 + b_1) W_2 + b_2,
\end{equation*}
where $\phi$ is the GELU activation function~\citep{hendrycks2016gelu},
$W_1 \in \mathbb{R}^{d_{\rm model} \times d_{\rm ff}}$ is the expansion matrix (\code{mlp\_c\_fc}), and
$W_2 \in \mathbb{R}^{d_{\rm ff} \times d_{\rm model}}$ is the projection matrix (\code{mlp\_c\_proj}).
In DistilGPT-2, $d_{\rm model} = 768$ and $d_{\rm ff} = 3072$.

\paragraph{Limitations of backward error frameworks for nonlinear activations.}
The backward error framework of \citet{beuzeville2025deterministic}
analyzes multi-layer networks by absorbing rounding errors into equivalent weight
perturbations $\Delta W$. For a layer followed by an activation $\phi$, the
componentwise backward error bound includes a term proportional to
$l_\phi / \kappa_\phi(z)$, where $\kappa_\phi(z) = |z \phi'(z) / \phi(z)|$ is the
relative condition number of $\phi$ and $l_\phi u$ bounds the evaluation error
of $\phi$.

For GELU, this framework does not yield a uniform backward guarantee due to the
vanishing derivative at its local minimum. Specifically, at $z^* \approx -0.752$,
$\phi'(z^*) = 0$ while $\phi(z^*) \approx -0.170 \neq 0$. Consequently, the
condition number vanishes ($\kappa_\phi(z^*) = 0$), causing the backward error bound
$l_\phi / \kappa_\phi(z)$ to diverge in the neighborhood of $z^*$.

Although each ideal SR step is conditionally unbiased, propagating SR perturbations
through a nonlinear activation generally introduces composition bias
\citep[Section~4]{elarar2024bounds}. To isolate the reduction error from this
nonlinear composition, we focus on the affine down-projection \code{mlp\_c\_proj} (and, by extension, the attention projection \code{attn\_c\_proj}):
\begin{equation*}
  y = h W + b,
\end{equation*}
where $h \in \mathbb{R}^n$ is the input row vector, $W \in \mathbb{R}^{n \times m}$ is the projection matrix,
$b \in \mathbb{R}^m$ is the bias vector, and $y \in \mathbb{R}^m$ is the output ($m = d_{\rm model} = 768$).
For \code{mlp\_c\_proj}, $n = d_{\rm ff} = 3072$ and $h = \phi(x W_1 + b_1)$; for \code{attn\_c\_proj}, $n = d_{\rm model} = 768$ and $h$ is the concatenated multi-head attention output.
Because these projection layers have no subsequent activation ($\phi = \mathrm{id}$), the forward
error $e = \hat{y} - y$ directly quantifies the perturbation injected into the
residual stream.

\paragraph{Error model.}
Each component of $y$ is evaluated as a left-to-right sequential inner product
followed by a bias addition:
\begin{equation*}
  y_i = \sum_{j=1}^{n} h_j \, W_{ji} + b_i, \qquad i = 1, \dots, m,
\end{equation*}
totaling $2n$ operations per output component with a longest error-product chain
of $n + 1$ rounding factors (one multiplication, $n - 1$ accumulations, and the
final bias addition).

Let $u = 2^{1-t}$ denote the relative ULP scale at virtual significand precision $t$.
Assuming no overflow, underflow, or NaNs occur, every elementary operation satisfies
$\fl(a \mathbin{\mathrm{op}} b) = (a \mathbin{\mathrm{op}} b)(1 + \delta)$, where:
\begin{itemize}[leftmargin=1.5em]
  \item under RN, $|\delta| \le u/2$;
  \item under SR (mode-1, nearness), $|\delta| \le u$, with the mean-independence
        property $\mathbb{E}[\delta_k \mid \delta_1, \dots, \delta_{k-1}] = 0$~\citep[Lemma~5.2]{connolly2021stochastic}.
\end{itemize}

\paragraph{Deterministic forward error bound (RN).}
Under deterministic RN, rounding errors may align in the worst case. Applying Higham's
exact forward error bound for inner products~\citep[Lemma~3.1]{higham2002accuracy}, each component
and then the vector satisfy
\begin{align*}
  |\hat{y}_i - y_i| &\le \left( (1 + u/2)^{n+1} - 1 \right) \left( \sum_{j=1}^{n} |h_j| |W_{ji}| + |b_i| \right), \\
  \|\hat{y} - y\|_2 &\le \left( (1 + u/2)^{n+1} - 1 \right) \, \| |h| \, |W| + |b| \|_2.
\end{align*}
In the high-precision regime where $nu \ll 1$, the coefficient linearizes to $\approx \frac{n+1}{2} u$;
when $nu \gg 1$, it grows exponentially, reaching $\approx 8.11 \times 10^{80}$ at $t = 4$ for $n=3072$.

\paragraph{Probabilistic rounding and scalar bounds.}
In the first-order regime, mean-independence allows SR forward error bounds
proportional to $\mathcal{O}(\sqrt{n} \, u)$ rather than the $\mathcal{O}(n \, u)$
deterministic RN bound. Existing probabilistic frameworks for bounding scalar
inner product errors $e_i = \hat{y}_i - y_i$ follow one of two strategies:
\begin{itemize}[leftmargin=1.5em]
  \item \emph{Azuma--Hoeffding (AH) method}~\citep{connolly2021stochastic,ipsen2020probabilistic,de2025error}:
        Martingale bounds provide exponential probability tails, which tolerate a union bound
        over dimensions with only a mild $\mathcal{O}(\sqrt{\ln m})$ penalty. However, the bound on
        the increments contains a factor $(1 + u)^{n+1}$~\citep[Section~7]{elarar2024bounds}.
        When $nu \gg 1$, this factor explodes exponentially,
        rendering AH uninformative at low precisions.
  \item \emph{Bienaym\'e--Chebyshev (BC) method}~\citep{elarar2022error,elarar2024bounds}:
        Bounding the variance of the accumulated error avoids this exponential explosion:
        because the precision enters squared in the variance bound,
        $(1 + u^2)^{n+1} - 1 \approx (n+1) u^2$ remains small even at low precisions where $nu \gg 1$.
\end{itemize}
The core variance bound for error products under SR is given by the following lemma.

\begin{lemma}[{\citealp[Lemma~3.1]{elarar2022error}}]
  \label[lemma]{lem:sr-variance}
  Let $\varphi = \prod_{k=1}^{K}(1 + \delta_k)$, where $\delta_k$ are SR-nearness rounding
  errors satisfying $|\delta_k| \le u$ and $\mathbb{E}[\delta_k \mid \delta_1, \dots, \delta_{k-1}] = 0$.
  Then $\mathbb{E}[\varphi] = 1$ and $\Var(\varphi) \le (1 + u^2)^K - 1$.
\end{lemma}

Applying Chebyshev's inequality to each scalar component $e_i$ yields an error bound scaling with $\frac{1}{\sqrt{\alpha}}\sqrt{(1 + u^2)^{n+1} - 1}$.

\paragraph{Vector $L_2$-norm bound: combining BC with a multivariate Markov bound.}
Extending the scalar analysis to the full output vector $e \in \mathbb{R}^m$ exposes a limitation in existing bounds: AH survives the union bound across $m$ components but fails at low precision $t$, whereas scalar BC remains stable at low precision but grows rapidly under the union bound. Because Chebyshev's inequality has a polynomial tail, setting the per-component failure probability to $\alpha / m$ inflates the confidence multiplier by $\sqrt{m}$, from $1/\sqrt{\alpha} \approx 4.47$ to $\sqrt{m/\alpha} \approx 124$ ($m = 768$, $\alpha = 0.05$). This penalty inflates the bound to $\approx 6870 \, u$, wiping out SR's advantage over RN.

To circumvent this limitation, our approach combines the low-precision robustness of BC with a single multivariate Markov step applied directly to the expected squared $L_2$-norm $\mathbb{E}[\|e\|_2^2]$. This bypasses the union bound, keeping the confidence multiplier $1/\sqrt{\alpha}$ independent of the output dimension $m$.

\begin{theorem}[Vector Bienaym\'e--Chebyshev Bound]
  \label{thm:vector-bc}
  Under SR-nearness, the forward error $e = \hat{y} - y$ of the projection $y = h W + b$
  satisfies, for any $\alpha \in (0, 1)$, with probability at least $1 - \alpha$:
  \begin{equation*}
    \|\hat{y} - y\|_2 \le \frac{1}{\sqrt{\alpha}} \sqrt{(1 + u^2)^{n+1} - 1} \, \| |h| \, |W| + |b| \|_2.
  \end{equation*}
  The confidence multiplier $1/\sqrt{\alpha}$ is independent of the output dimension $m$.
\end{theorem}

\begin{proof}
  The error in output component $i$ is
  \begin{equation*}
    e_i = \sum_{j=1}^{n} h_j W_{ji} (\varphi_{ij} - 1) + b_i (\varphi_{b,i} - 1),
  \end{equation*}
  where $\varphi_{ij}$ and $\varphi_{b,i}$ represent the error products along the accumulation
  chains, containing at most $n + 1$ and $1$ rounding factors, respectively. By mean-independence,
  $\mathbb{E}[\varphi_{ij}] = 1$, $\mathbb{E}[\varphi_{b,i}] = 1$, and $\mathbb{E}[e_i] = 0$.

  Although the terms $\varphi_{ij} - 1$ share common rounding errors across additions,
  Minkowski's inequality in $L_2$ bounds the standard deviation of their linear combination:
  \begin{equation*}
    \sqrt{\Var(e_i)} \le \sum_{j=1}^{n} |h_j| |W_{ji}| \sqrt{\Var(\varphi_{ij})} + |b_i| \sqrt{\Var(\varphi_{b,i})}.
  \end{equation*}
  Applying \cref{lem:sr-variance} with $\Var(\varphi_{ij}) \le (1 + u^2)^{n+1} - 1$ and
  $\Var(\varphi_{b,i}) \le u^2 \le (1 + u^2)^{n+1} - 1$ yields
  \begin{equation*}
    \Var(e_i) \le \left( (1 + u^2)^{n+1} - 1 \right) \left( \sum_{j=1}^{n} |h_j| |W_{ji}| + |b_i| \right)^2.
  \end{equation*}
  By linearity of expectation, the expected squared $L_2$-norm satisfies
  \begin{equation*}
    \mathbb{E}[\|e\|_2^2] = \sum_{i=1}^{m} \Var(e_i) \le \left( (1 + u^2)^{n+1} - 1 \right) \, \| |h| \, |W| + |b| \|_2^2.
  \end{equation*}
  Applying Markov's inequality to the squared $L_2$-norm $\|e\|_2^2$ (which constitutes the
  multivariate Bienaym\'e--Chebyshev inequality) gives
  $\mathbb{P}\left( \|e\|_2^2 \ge \frac{1}{\alpha} \mathbb{E}[\|e\|_2^2] \right) \le \alpha$.
  Taking the square root completes the proof.
\end{proof}

\paragraph{Numerical evaluation on DistilGPT-2.}
While \cref{thm:vector-bc} holds for any affine projection of reduction length $n$,
we evaluate the bounds concretely on the feedforward down-projection \code{mlp\_c\_proj}
($n = d_{\rm ff} = 3072$, $m = d_{\rm model} = 768$) at 95\%
confidence ($\alpha = 0.05$, $1/\sqrt{\alpha} = \sqrt{20} \approx 4.47$).
Because \code{mlp\_c\_proj} has four times the reduction length of \code{attn\_c\_proj} ($n = 768$),
it has the longest accumulation chains and therefore tests the error bounds in the most demanding regime.
In the limiting first-order regime $u\to0$ (equivalently, $t\to\infty$):
\begin{itemize}[leftmargin=1.5em]
  \item RN deterministic: $(1 + u/2)^{n+1} - 1 \approx \frac{n+1}{2} u \approx 1537 \, u$;
  \item SR vector BC: $\frac{1}{\sqrt{\alpha}} \sqrt{(1 + u^2)^{n+1} - 1} \approx \sqrt{20} \sqrt{n+1} \, u \approx 248 \, u$.
\end{itemize}
The SR coefficient is asymptotically smaller than the RN coefficient by a factor
of approximately $6.2$. At finite precision the exact ratio approaches this
limit from above: it is $14.37$ at $t=11$ and $6.35$ at $t=16$.

\begin{figure}[t]
  \centering
  \safeincludegraphics[width=0.75\linewidth]{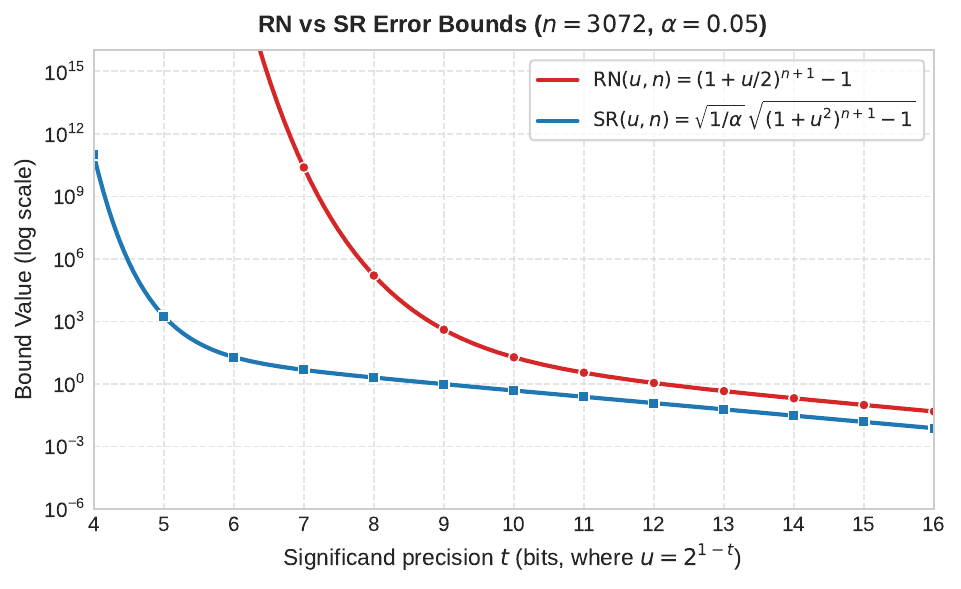}
  \caption{Deterministic RN and SR forward-error bound
  coefficients for \code{mlp\_c\_proj} in DistilGPT-2 ($n=3072$).}
  \label{fig:validity-bounds}
\end{figure}

As shown in \cref{fig:validity-bounds}, the deterministic RN bound suffers from exponential error
compounding at low precision ($nu \gg 1$), reaching $2.43 \times 10^{10}$ at $t = 7$ and $8.11 \times 10^{80}$ at $t = 4$.
In contrast, because precision enters squared under stochastic rounding, the SR
vector BC bound grows much more slowly ($4.73$ at $t = 7$). At $t=11$ the exact RN and SR
coefficients are $3.482$ and $0.242$, a ratio of $14.37$; the ratio decreases
toward $6.2$ only in the limit $t\to\infty$.

For fixed $\alpha$ in the first-order regime, \cref{thm:vector-bc} yields an
$\mathcal{O}(\sqrt{n}u)$ bound for SR, whereas deterministic RN scales as
$\mathcal{O}(nu)$. Both bounds are worst-case over the inputs and normalize by
$\| |h| \, |W| + |b| \|_2$, but they treat arithmetic errors
differently. Under RN, rounding errors are deterministic functions of the data
that can align constructively in the worst case. Under SR, mean-independence
yields the $\sqrt{n}$ factor via the Bienaym\'e--Chebyshev variance bound. While
both results are upper bounds, and realized RN errors rarely achieve worst-case
alignment, the comparison illustrates the fundamental contrast between the two modes:
deterministic rounding permits systematic linear drift along the reduction chain,
whereas stochastic rounding restricts error growth to a sub-linear
$\mathcal{O}(\sqrt{n}u)$ envelope.

The downstream captures in \cref{sec:analysis-empirical-drift} show larger RN than SR
sample logit drift after reducing \code{mlp\_c\_proj}. These propagated
measurements do not directly test the single-reduction bounds above.

%% file: analysis_2.tex
\FloatBarrier
\subsection{Error propagation and loss decomposition at the output softmax}
\label{sec:analysis-softmax}

\Cref{sec:analysis-mlp} bounds the perturbation that low-precision arithmetic
injects into one layer.  At the network output, the error travels to the logits
and alters each token's cross-entropy loss.
This subsection analyzes that interface to identify, term by term,
why SR and RN affect the language-model head differently.

\paragraph{Exact decomposition of the cross-entropy loss change.}
Fix one token with reference logits $\lambda\in\R^V$, target class~$y$, and
predictive distribution $p=\softmax(\lambda)$.  Its cross-entropy loss is
\begin{equation}
  \mathcal L(\lambda)=-\lambda_y+\log\sum_{j=1}^Ve^{\lambda_j}.
  \label{eq:token-cross-entropy}
\end{equation}
Let the low-precision run perturb logit $j$ by~$\Delta_j$.  Substituting
$\lambda+\Delta$ into \cref{eq:token-cross-entropy} and subtracting gives
the exact signed loss change
\begin{equation*}
  \delta\mathcal L(\Delta)
  :=\mathcal L(\lambda+\Delta)-\mathcal L(\lambda)
  =-\Delta_y+\log\sum_{j=1}^Vp_je^{\Delta_j}.
\end{equation*}
Let $\widetilde p=\softmax(\lambda+\Delta)$ and let
$D_{\mathrm{KL}}(p\,\|\,\widetilde p)$ denote the Kullback--Leibler (KL)
divergence from $p$ to $\widetilde p$. Direct substitution gives
\begin{equation*}
  D_{\mathrm{KL}}(p\,\|\,\widetilde p)
  =\sum_{j=1}^V p_j\log\frac{p_j}{\widetilde p_j}
  =\log\sum_{j=1}^V p_j e^{\Delta_j}-p^\top\Delta,
\end{equation*}
and therefore, with $g:=p-e_y$, the exact decomposition
\begin{equation}
  \delta\mathcal L(\Delta)
  =g^\top\Delta+D_{\mathrm{KL}}(p\,\|\,\widetilde p).
  \label{eq:one-token-kl-decomposition}
\end{equation}
The linear term is signed and records the target-specific shift, whereas the
KL term is non-negative and records the change in the full predictive
distribution.

\paragraph{Local quadratic form of the KL term.}
The KL divergence satisfies $D_{\mathrm{KL}}(p\,\|\,p)=0$ and has vanishing gradient
at $\Delta=0$. Its local quadratic approximation is governed by the Fisher
information matrix of the categorical distribution parameterized by logits,
\begin{equation*}
  H = \diag(p) - pp^\top,
\end{equation*}
which coincides with the Hessian of the cross-entropy loss $\nabla^2_\lambda \mathcal L(\lambda)$
\citep[Section~6]{martens2020new}. Although the KL divergence is asymmetric, both
orientations share the same second-order expansion at $\Delta=0$:
\begin{equation}
  D_{\mathrm{KL}}(p\,\|\,\widetilde p)
  =\tfrac12\,\Delta^\top H\Delta+r,
  \qquad r=\mathcal O(\|\Delta\|^3).
  \label{eq:logit-kl-second-order}
\end{equation}
Because $H$ is the covariance matrix of a single categorical draw with probabilities $p$,
for any perturbation $\Delta \in \R^V$ the quadratic form equals the variance across classes:
\begin{equation}
  \Delta^\top H \Delta = \sum_{j=1}^V p_j \Delta_j^2 - (p^\top\Delta)^2 = \Var_p(\Delta) \ge 0,
  \label{eq:fisher-variance-identity}
\end{equation}
confirming that $H \succeq 0$ and that the local KL penalty charges variations in
$\Delta$ across classes considered plausible by $p$.
Fisher information similarly guides parameter sensitivity in post-training
quantization (e.g., SqueezeLLM, \citealp{kim2024squeezellm}); in logit space,
$H$ arises analytically to penalize non-uniform operational perturbations directly.

An exact consequence of the logit formulation is \emph{shift invariance}
\citep{blanchard2021accurately}: any uniform perturbation $\Delta = c\bfone$ leaves
the predictive distribution identical ($\widetilde p = \softmax(\lambda + c\bfone) = p$),
so that $D_{\mathrm{KL}}(p\,\|\,\widetilde p) = 0$.
Likewise, $g^\top (c\bfone) = c(p-e_y)^\top\bfone = 0$, ensuring that $\delta\mathcal L(c\bfone) = 0$
exactly. The one-dimensional subspace $\operatorname{span}(\bfone)$ constitutes the null
space of the Fisher metric ($H\bfone = \mathbf{0}$).

\paragraph{The two probability spaces.}
Conditional on a fixed token~$w$, two distinct probability spaces arise for SR
and RN.
\begin{enumerate}[leftmargin=1.7em,label=(\roman*)]
  \item Under SR, every rounded elementary operation along the trajectory draws
        fresh pseudorandom bits. This operational randomness makes the
        logit perturbation $\Delta_{\SR}^{(w)}$ a random vector, driven by
        elementary rounding errors that satisfy the mean-independence of \cref{sec:analysis-mlp}.

  \item Under RN, every operation returns one deterministic value: there is no
        operational randomness. Conditional on a fixed token,
        \begin{equation}
          \E_{\RN}[\Delta_{\RN}^{(w)}\mid w]=\Delta_{\RN}^{(w)},
          \qquad
          \Sigma_{\RN}^{(w)}:=\operatorname{Cov}_{\RN}(\Delta_{\RN}^{(w)}\mid w)=0.
          \label{eq:rn-degenerate}
        \end{equation}
\end{enumerate}

\paragraph{Drift--variance decomposition.}
For either mode $m\in\{\SR,\RN\}$, decompose the logit perturbation as
\begin{equation}
  \Delta_m=b_m+\xi_m,
  \qquad
  \E_m[\xi_m]=0,
  \label{eq:local-bias-cov}
\end{equation}
where $b_m:=\E_m[\Delta_m]$ is the drift and $\xi_m$ is the centered
fluctuation, with covariance $\Sigma_m:=\operatorname{Cov}_m(\xi_m)$.
By \cref{eq:rn-degenerate}, on a fixed token RN carries
$b_{\RN}=\Delta_{\RN}$ and $\Sigma_{\RN}=0$: its entire perturbation is
drift.

Substituting $\Delta_m=b_m+\xi_m$ into the local KL expansion
\cref{eq:logit-kl-second-order} and taking expectations $\E_m[\cdot]$ over
operational randomness (over rounding draws for SR, and degenerate for
deterministic RN), the linear term in \cref{eq:one-token-kl-decomposition} yields
$\E_m[g^\top\Delta_m]=g^\top b_m$.
The quadratic term expands as
\begin{equation*}
  \Delta_m^\top H\Delta_m
  =b_m^\top Hb_m + 2b_m^\top H\xi_m + \xi_m^\top H\xi_m.
\end{equation*}
Taking expectations, the cross term vanishes since $\E_m[\xi_m]=0$, and the
quadratic fluctuation term evaluates to $\E_m[\xi_m^\top H\xi_m] = \operatorname{tr}(H\Sigma_m)$.
Thus the KL contribution itself separates as
\begin{equation*}
  \E_m[D_{\mathrm{KL}}(p\,\|\,\widetilde p_m)]
  =\tfrac12\,b_m^\top Hb_m +\tfrac12\operatorname{tr}(H\Sigma_m) +R_m,
\end{equation*}
where $R_m:=\E_m[r]$.  Combining this KL decomposition with the linear
target-shift term yields the expected excess loss:
\begin{equation}
  \E_m[\delta\mathcal L]
  =\underbrace{g^\top b_m}_{D_m:\text{ drift}}
  +\underbrace{\tfrac12\,b_m^\top Hb_m}_{Q_m:\text{ curvature of drift}}
  +\underbrace{\tfrac12\operatorname{tr}(H\Sigma_m)}_{V_m:\text{ variance penalty}}
  +\underbrace{R_m}_{\text{remainder}},
  \label{eq:head-full-comparison}
\end{equation}
where $D_m$ is signed; $Q_m$ and $V_m$ are non-negative by $H\succeq0$.
Removing drift is not free: an unbiased mode avoids the drift terms $D_m+Q_m$,
but pays the variance penalty $V_m$.

\subsection{Empirical validation of the drift--variance trade-off}
\label{sec:analysis-empirical-drift}

We evaluate this trade-off on DistilGPT-2 using $N=1020$ scored WikiText-2
positions, with reference loss $\mathcal L_0=4.1354$ nats.
The decomposition uses $32$ SR seeds at $t=4,\ldots,10$, with one
deterministic RN run per precision. We use bias-corrected quadratic
estimators.  The logit measurement protocol and the estimators are detailed in \cref{app:drift-measurements}.

\paragraph{SR avoids systematic drift but incurs local fluctuation.}
The decomposition $\Delta_m = b_m + \xi_m$ (\cref{eq:local-bias-cov})
splits each logit perturbation into drift $b_m = \E_m[\Delta_m]$ and centered
fluctuation $\xi_m$. For RN the split is trivial: $b_{\RN} = \Delta_{\RN}$
and $\xi_{\RN} = 0$, so the perturbation at each token--coordinate pair is
entirely deterministic drift. For SR, the perturbation is stochastic,
so its expected magnitude is measured by its operational root mean square.
All plotted quantities represent the operational statistics of a
single forward pass; multiple seeds ($S=32$) are used exclusively to construct
empirical estimators for these quantities.

\Cref{fig:drift-distribution} visualizes the three quantities that enter the
loss decomposition (\cref{eq:head-full-comparison}). Row~(i) shows the
operational perturbation magnitude per run across all token--coordinate pairs.
At the MLP down-projection, RN accumulates deterministic drift, whereas SR's
per-run perturbation is smaller. At the head, however, SR's per-run perturbation
is larger than RN's. This does not contradict the worst-case bounds of
\cref{thm:vector-bc}. At the head, inputs are centered by LayerNorm;
empirically, we observe that this centering induces cancellation in RN rounding
errors, preventing the accumulation of systematic drift. Because the projection
is linear, SR guarantees unbiasedness, but incurs higher local variance to
do so (local errors up to $1.0u$, compared to RN's strict $0.5u$ bound). In the
absence of systematic drift, the accumulated perturbation is governed by this
local variance, which is consistent with SR perturbations being larger than RN
at the head.

Nevertheless, not all of this perturbation magnitude translates to excess loss.
By shift invariance (\cref{eq:fisher-variance-identity}), a per-token uniform
shift leaves softmax probabilities unchanged, so only the non-uniform
component feeds into $Q_m$. Row~(ii) isolates this loss-relevant component
by centering the drift (subtracting each token's vocabulary mean from $b_j$).
The reduction is pronounced at the MLP: RN's spread shrinks by an order
of magnitude, confirming that most of its accumulation was uniform.
SR's centered drift is near-zero everywhere.

Row~(iii) removes the uniform component of the fluctuation (to which the
softmax is insensitive by shift invariance). Let $\tilde\xi_j = \xi_j - p^\top\xi$
denote the $p$-centered operational fluctuation; the row reports its standard
deviation $\mathrm{std}(\tilde\xi_j)$. RN is deterministic
($\xi_{\RN}=0$). The centered fluctuation is smaller at the MLP than at
the head.
MLP noise is predominantly a uniform shift of the logits, whereas
head noise changes the differences between logits and therefore the softmax
output.
\begin{figure}[!htb]
  \centering
  \safeincludegraphics[width=\linewidth]{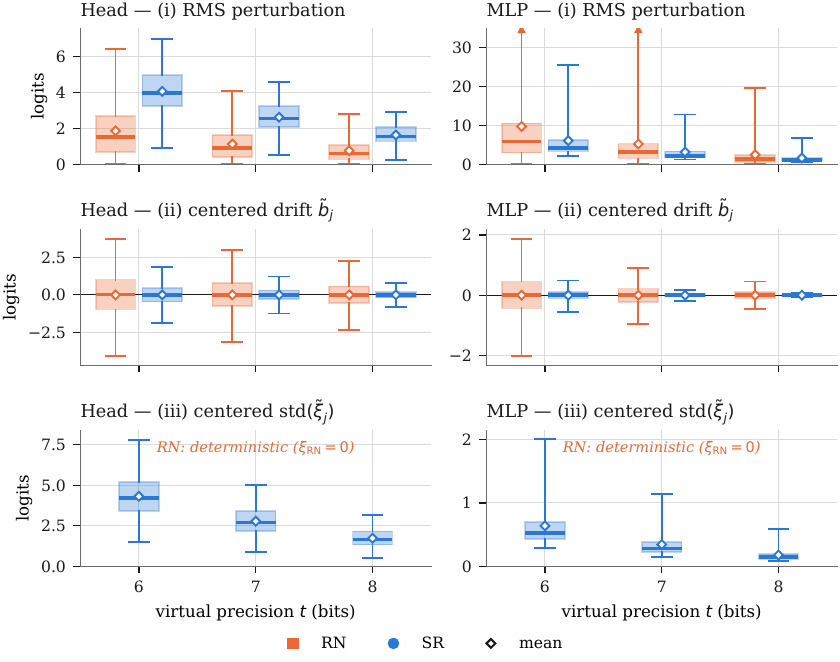}
  \vspace{-1em}
  \caption{Operational decomposition of logit perturbations across all $51{,}262{,}140$ token--coordinate pairs (for the MLP, rounding errors are injected at the down-projection and propagated to the final logits).
  \textbf{(i)} Per-run RMS perturbation for pair $j$: $|\Delta_j|$ for deterministic RN ($\xi_j=0$), and $\sqrt{\mathbb{E}[\Delta_j^2]} = \sqrt{b_j^2 + \operatorname{std}(\xi_j)^2}$ for SR.
  \textbf{(ii)} Centered drift $\tilde b_j$ (per-token vocabulary mean removed): non-uniform component charged via $Q_m$.
  \textbf{(iii)} Operational standard deviation of the $p$-centered fluctuation $\mathrm{std}(\tilde\xi_j)$, where $\tilde\xi_j = \xi_j - p^\top\xi$; RN: $\xi_{\RN}=0$.
  Boxes IQR, median line, 1st--99th-percentile whiskers; vertical scales differ.}
  \label{fig:drift-distribution}
\end{figure}

\paragraph{The role of input distributions.}
The accumulation of deterministic rounding errors in inner products depends heavily on the distribution of the operands \citep{elarar2022error}. When inputs are predominantly positive (such as the outputs of the GELU activation at the MLP down-projection), RN rounding errors become biased and, depending on the weight distribution, systematically accumulate into a large drift. Conversely, when inputs are zero-centered (such as the outputs of LayerNorm at the language-model head), positive and negative RN errors tend to cancel, preventing the accumulation of systematic drift.

\paragraph{Loss decomposition at the language-model head.}
\label{par:head-cost-unbiasedness}
Although SR produces smaller logit drift at the head, it incurs higher excess
loss than RN throughout $t=4,\ldots,10$
(\cref{tab:decomposition,fig:head-decomposition-terms}). The decomposition in
\cref{eq:head-full-comparison} explains the reversal. RN is deterministic
($V_{\RN}=0$), so its perturbation contributes only through $D_{\RN}$,
$Q_{\RN}$, and $R_{\RN}$. SR's operational randomness introduces non-uniform
fluctuations---the large $p$-centered fluctuation visible in
\cref{fig:drift-distribution}~(iii)---that incur the quadratic cost
$V_{\SR}=\tfrac12\operatorname{tr}(H\Sigma_{\SR})$. At the head the noise is
genuinely non-uniform across the vocabulary, so the softmax cannot ignore it.

Over $t=6,\ldots,9$, the variance penalty $V_{\SR}$ accounts for $100$--$105\%$ of SR's
excess loss, while $D_{\SR}$ and $Q_{\SR}$ lie within $1.6$ standard errors of
zero. For RN, the perturbation carries no fluctuation ($V_{\RN}=0$);
its non-uniform drift produces $Q_{\RN}=1.58$ and $R_{\RN}=0.50$
nats at $t=6$, for a total excess of $2.12$ nats---less than half of SR's
$4.40$ nats.

\paragraph{Loss decomposition at the MLP down-projection.}
At the MLP down-projection, the loss ordering reverses:
SR has an excess loss of $0.05$ nats against RN's $0.24$ at $t=6$
(\cref{fig:mlp-decomposition-terms,tab:decomposition}).
The 3072-length inner product allows deterministic rounding errors under RN
to align across many terms, producing a wide, non-uniform drift dispersion
that incurs a curvature penalty of $Q_{\RN}=0.29$ nats.

SR prevents this systematic alignment, keeping $Q_{\SR}$ at $0.017$ nats.
Although SR's raw fluctuation is comparable to or larger than at the head
($6.23$ versus $4.27$ logits RMS at $t=6$), the noise has far fewer degrees of
freedom. At the head, SR injects rounding errors directly into all $50{,}257$
logits, scrambling the differences between them. At the MLP down-projection, SR
injects noise into the $768$-dimensional hidden state. Restricted to these $768$
degrees of freedom, the $50{,}257$ logits are highly coupled.  This
dimensionality restriction likely explains why the MLP fluctuation manifests
predominantly as a uniform shift across the vocabulary, which the softmax simply
ignores. The near-uniform shift also depends on the trained output weights and
the directions of the propagated error. As a result, the $p$-centered
fluctuation in \cref{fig:drift-distribution}(iii) collapses at the MLP relative
to the head, yielding $V_{\SR}=0.083$ nats, fifty times smaller than at the
head. The reduction in drift curvature therefore more than compensates for SR's
variance penalty.

\begin{figure}[!htb]
  \centering
  \safeincludegraphics[width=\linewidth]{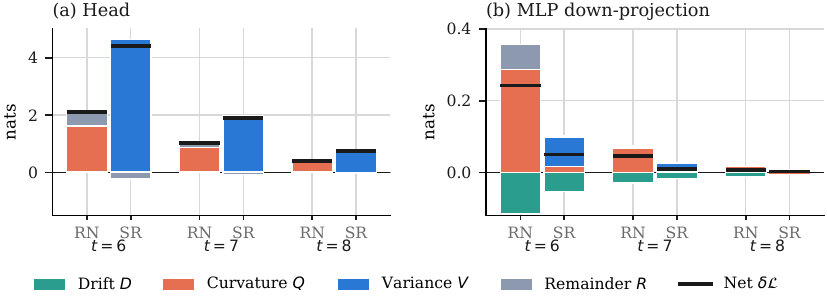}
  \caption{Stacked decomposition of expected excess loss
    $\E_m[\delta\mathcal L] = D_m + Q_m + V_m + R_m$ at the head and MLP
    down-projection at $t=6,7,8$. Positive components stack upward from zero
    and negative components downward; solid horizontal marks indicate the
    measured mean excess loss (with one standard error for SR; RN is
    deterministic). At the head, the variance penalty $V$
    dominates SR excess loss; at the MLP down-projection, RN is dominated by drift
    curvature $Q$. Full values are in \cref{tab:decomposition}.}
  \label{fig:head-decomposition-terms}
  \label{fig:mlp-decomposition-terms}
\end{figure}

%% file: experiments.tex
\section{Experiments}
\label{sec:experiments}

\Cref{sec:analysis} bounds the forward error of a single reduction and
decomposes its effect on the loss at the output softmax. Neither result
predicts the perplexity ordering of SR and RN once rounding errors propagate
through the network: the bounds are worst-case and per reduction, and the loss
decomposition is evaluated one site at a time. This section measures that
ordering on DistilGPT-2 at matched significand precision, with the scope of
reduced precision ranging from the whole network to one projection in one block
(\cref{tab:sweeps}). The sweeps address three questions: which rounding rule
gives lower perplexity at each site, how per-site effects combine across sites
and across depth, and whether the per-site preferences yield a mixed-rounding
assignment that improves on RN at identical per-site precisions.

\subsection{Experimental setup}
\label{sec:exp-setup}

\paragraph{Model and evaluation benchmark.}
We evaluate the public \code{distilbert/distilgpt2} checkpoint
\citep{distilgpt2modelcard,wolf2020transformers}. The architecture retains the
layer widths of GPT-2 small ($d_{\rm model} = 768$, $d_{\rm ff} = 3072$, $H =
  12$ heads, $d_{\rm head} = 64$, vocabulary $V = 50257$, as illustrated in
\cref{fig:architecture}) and distills the depth to six blocks. It therefore has
the reduction lengths that enter the bounds of \cref{sec:analysis-mlp}:
$n=768$ for \code{attn\_c\_attn}, \code{attn\_c\_proj}, and \code{mlp\_c\_fc},
and $n=3072$ for \code{mlp\_c\_proj}. Six blocks keep per-operation software
emulation tractable. All evaluations use the pretrained weights without
fine-tuning, calibration, or quantization-aware adjustment, so that
configurations differ only in the arithmetic under test.

Evaluation uses the WikiText-2 test split \citep{merity2017pointer}. The harness
selects the leading $1\%$ of lines, joins them with double newlines, tokenizes
the resulting text, and evaluates its first $1024$ tokens. These are divided
into four non-overlapping contexts of $256$ tokens, with context reset between
them. Each context contributes $255$ next-token targets, giving $1020$ scored
positions. Perplexity is the exponential of the mean cross-entropy over these
positions; the PRISM RN reference at \code{binary32} significand precision
($t=24$) is $\mathrm{PPL}_0 = 62.51$. When effects are compared or summed, we
also report the excess loss $\ell = \ln(\mathrm{PPL}/\mathrm{PPL}_0)$ in nats,
which equals the increase in mean cross-entropy over the $1020$ positions and is
the quantity decomposed in \cref{eq:head-full-comparison}.

\paragraph{Instrumented arithmetic and scoping.}
\label{sec:exp-scoping}
Inference is executed within Fuzzy PyTorch \citep{gonzalezpepe2026fuzzy}. During
the framework's compilation, the Verificarlo compiler
pass~\citep{denis2016verificarlo} intercepts native floating-point operations
and redirects them to PRISM's vectorized C++ backend (\cref{sec:prism}).
Instrumented operations within the target scope are rounded to a virtual
significand precision $t \le 24$ using either SR or untied RN, in which half-ULP
ties round away from zero rather than to even (\cref{sec:prism}). All SR--RN
comparisons in this section use untied RN and not IEEE round-to-nearest with
ties-to-even (RNE) at reduced precision. The two differ only on exact ties. RN
introduces no randomness, and repeated RN runs are bitwise reproducible.

The harness attaches PyTorch forward pre-hooks that set precision $t$ and the
rounding mode immediately before the target module executes, and forward
post-hooks that restore $t = 24$ under RN on exit. Inference runs in PyTorch
eager mode without cross-module operator fusion. Every operation outside the
target scope, including the final loss and perplexity computation, runs at
$t=24$ under RN, so an SR run and its paired RN control differ only in the
rounding rule applied within the target scope.

The component and sublayer sweeps reduce precision at the target in all six
blocks at once; the depth sweeps reduce it in one block or in a prefix of
blocks (\cref{tab:sweeps}). The \code{attention} and \code{mlp} scopes cover
every operation inside the corresponding module of each block: the query--key--value
and output projections, the score and value products, and the softmax for
\code{attention}; \code{mlp\_c\_fc}, the GELU, and \code{mlp\_c\_proj} for
\code{mlp}. The \code{attn\_qkav} scope covers only the products $QK^\top$
(reduction length $d_{\rm head}=64$) and $AV$ (at most $T=256$), with the
softmax at $t=24$. The \code{lm\_head} scope is the projection onto the
vocabulary logits.

\begin{table}[t]
  \centering
  \caption{Sweep levels. All levels compare SR and untied RN at matched
    significand precision $t$; operations outside the target scope run at
    $t=24$ under RN.}
  \label{tab:sweeps}
  \small
  \begin{tabular}{@{}l >{\raggedright\arraybackslash}p{4.2cm} c >{\raggedright\arraybackslash}p{4.8cm}@{}}
    \toprule
    Level & Target scope                                                                  & $t$ swept & Question \\
    \midrule
    Global
          & Whole model
          & $4\text{--}14$
          & SR vs. RN comparison under uniform precision                                                         \\
    Component
          & \code{attention}, \code{mlp}, \code{lm\_head}, \code{attn\_qkav} (all blocks)
          & $4\text{--}14$
          & Site-dependent sensitivity across main modules                                                       \\
    Sublayer
          & Four projections \code{attn\_*}, \code{mlp\_*} (all blocks)
          & $4\text{--}14$
          & Impact of reduction length ($n=768$ vs.\ $3072$)                                                     \\
    Depth
          & \code{mlp\_c\_proj} (single blocks; prefixes $0\text{--}k$)
          & $4,5,6$
          & Dependence on position; additivity of per-block effects                                              \\
    Mixed recipe
          & SR at output projections, RN elsewhere
          & $6,7,8$
          & Mixed rounding versus all-RN at identical per-site precisions                                        \\
    \bottomrule
  \end{tabular}
\end{table}

\paragraph{Statistical protocol.}
The global, component, sublayer, and depth sweeps use five independent seeds per
SR configuration, and the mixed-rounding comparison in
\cref{sec:exp-composition} uses five seeds per mixed configuration. Results are
reported as sample mean $\pm$ standard deviation; for SR, excess loss is
averaged over seeds. RN is deterministic at a fixed execution configuration and
is run once. We define the perplexity gap as $\Delta = \overline{\SR} - \RN$, so
negative values favor SR. We regard the gap as statistically distinguishable from
zero when the RN value falls outside the two-sided $95\%$ Student-$t$ confidence
interval for the SR mean ($df=4$). The evaluated text is fixed, so every
interval is conditional on the $1020$ scored positions and does not describe
variation across texts.

\subsection{Results}
\label{sec:exp-results}

The experiments yield four findings:
\begin{enumerate}[label=(\roman*),leftmargin=2em]
  \item \textbf{RN is better under uniform rounding.}
        Reducing only the language-model head reproduces most of the
        SR--RN gap, suggesting that the head drives the global preference
        for RN (\cref{fig:global,fig:component}).

  \item \textbf{The preferred rounding rule depends on the site.}
        SR gives lower perplexity in the MLP, whereas RN is preferred at
        the head and attention input projection. The loss decomposition
        attributes the MLP advantage to reduced drift and the head
        disadvantage to SR fluctuations
        (\cref{fig:component,fig:sublayer}).

  \item \textbf{RN loss compounds more strongly across blocks.}
        For \code{mlp\_c\_proj} at $t=5$--$6$, RN excess loss exceeds
        the sum of single-block losses more than SR excess loss does.
        At $t=6$, SR is approximately additive
        (\cref{fig:cumulative}).

  \item \textbf{Mixed-rounding improves on full RN.}
        Using SR at the attention and MLP output projections and RN
        elsewhere lowers perplexity, even with one fewer significand
        bit at the MLP down-projection
        (\cref{tab:accuracy-matched-precision}).
\end{enumerate}

\paragraph{Global uniform-precision sweep.}
Under uniform precision, RN consistently outperforms SR across the entire
low-to-moderate precision range, with statistically significant differences at
every $t \in [4,11]$ (\cref{fig:global}). At $t=8$, perplexity is $97.68$ under
RN and $139.27\pm3.58$ under SR. Reducing only the head reproduces this gap: the
SR--RN difference in excess loss is $2.03$ nats globally and $2.27$ nats with
only \code{lm\_head} reduced at $t=6$, and $0.35$ against $0.37$ nats at $t=8$.
The RN advantage is therefore consistent with the head's SR fluctuation penalty
(\cref{sec:analysis-empirical-drift}) dominating the whole-network result. This
general performance agrees with the forward-pass preference for RN reported by
\citet{chmiel2023accurate}, establishing RN as the appropriate default for
global precision reduction.

\begin{figure}[t]
  \centering
  \safeincludegraphics{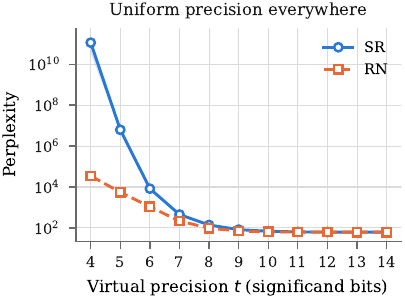}
  \caption{Global sweep at uniform precision across all model operations, showing mean perplexity and SR standard deviation on a logarithmic axis. RN's advantage is statistically significant across $t\in[4,11]$.}
  \label{fig:global}
\end{figure}

\begin{figure}[t]
  \centering
  \safeincludegraphics{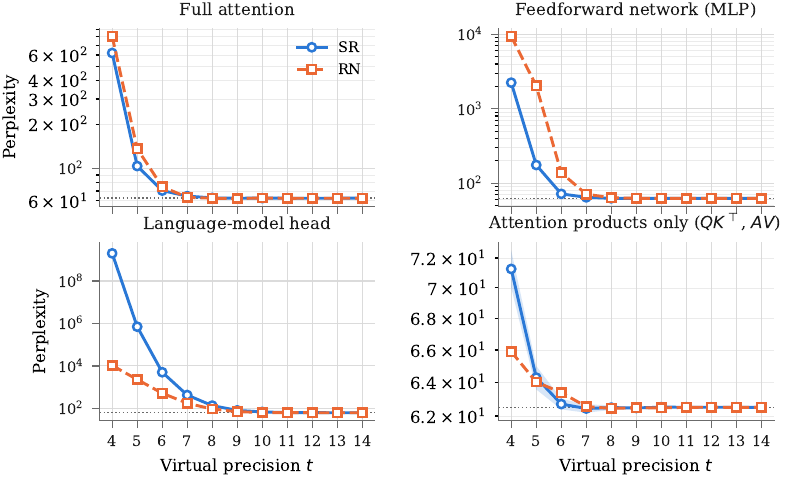}
  \caption{Per-component sweep across precisions $t \in [4, 14]$. Bands show $\pm 1$ sample SD over five SR seeds; dotted line indicates the full-precision reference ($62.51$).}
  \label{fig:component}
\end{figure}

\paragraph{Site performance: MLP versus language-model head.}
In the MLP, SR gives lower perplexity than RN at every $t\le8$, with
statistically significant differences (\cref{fig:component}). At $t=6$,
perplexity is $71.97\pm1.16$ under SR and $137.93$ under RN, a degradation of
$15\%$ against $121\%$ over the reference. At $t=8$, SR is within $0.3\%$ of the
reference ($62.69$) and RN within $2.0\%$ ($63.75$). At $t\le5$ both rules more
than double the reference perplexity (at $t=5$, $175.65$ under SR against
$2029.83$ under RN); the ordering there is statistically significant but
concerns configurations that neither rule makes usable.

At the language-model head the ordering reverses: RN gives lower perplexity at
every $t\in[4,12]$, with statistically significant differences. At $t=8$,
perplexity is $135.56\pm5.22$ under SR and $93.60$ under RN; at $t=6$,
$5054\pm464$ against $519.18$. \Cref{sec:analysis-empirical-drift} decomposes
the head's excess loss at these precisions. Over $t=6,\ldots,9$ the fluctuation
term $V_{\SR}$ accounts for $100\%$ of SR's excess loss, whereas RN,
which has no fluctuation, pays only for its non-uniform drift ($2.12$ against
$4.40$ nats at $t=6$). SR noise injected at the head changes the differences
between logits, which the cross-entropy penalizes; SR noise injected in the MLP
reaches the logits predominantly as a uniform shift, to which the softmax is
invariant.

The \code{attention} module favors SR at $t\le6$ but do not favor any rounding
mode at higher precisions; at $t\ge7$ the two rules differ by less than $2\%$ of
the reference (at $t=7$, $64.60\pm0.49$ under SR against $63.42$ under RN).
Restricted to the score and value products (\code{attn\_qkav}), reduced
precision raises perplexity by at most $14\%$ at $t=4$ ($71.26\pm1.26$ under SR,
$65.90$ under RN) and by at most $2.9\%$ at $t\ge5$. These products have the
shortest reductions in the model ($d_{\rm head}=64$ and at most $T=256$ terms),
which may explain their small sensitivity to either rule.

The \code{attention} module does not decompose into its parts. At $t=4$, the
excess losses of \code{attn\_c\_attn}, \code{attn\_c\_proj}, and
\code{attn\_qkav}, each reduced separately, sum to $1.20$ nats under RN and
$1.66$ nats under SR, which favors RN; reducing the whole module gives $2.56$
and $2.29$ nats, which favors SR. At $t=6$ the sums are equal under both rules
($0.094$ nats), whereas the whole module gives $0.182$ under RN and $0.120$
under SR.

\begin{figure}[t]
  \centering
  \safeincludegraphics{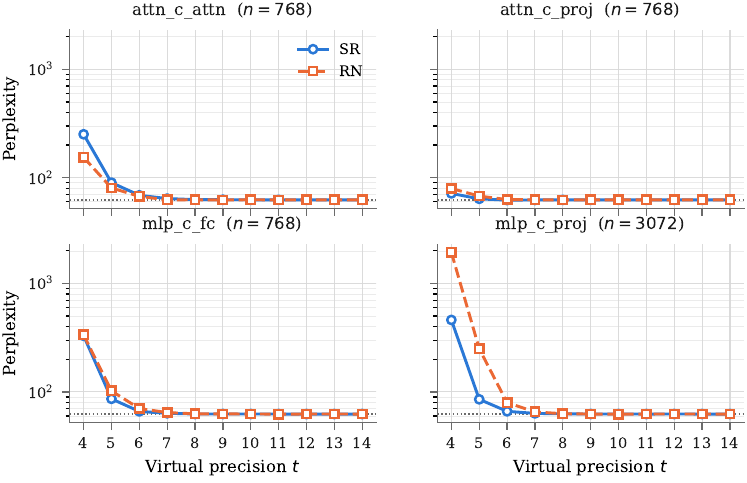}
  \caption{Per-sublayer sweep across the four primary projection matrices, with
    reduction lengths $n$ indicated and perplexity on a shared logarithmic scale.
    The feedforward down-projection (\code{mlp\_c\_proj}, $n=3072$) exhibits the
    largest low-precision SR advantage ($\Delta = -1473.15$ at $t=4$).}
  \label{fig:sublayer}
\end{figure}

\paragraph{Reduction length and sublayer sensitivity.}
\Cref{fig:sublayer} separates the four projections, each reduced in all six
blocks. The MLP down-projection (\code{mlp\_c\_proj}, $n=3072$) has the largest
SR advantage, statistically significant at every $t\le7$: at $t=6$, perplexity
is $66.24\pm0.63$ under SR and $79.64$ under RN, a degradation of $6.0\%$
against $27\%$. Among the three projections with $n=768$, the preferred rule
depends on the site. The attention input projection (\code{attn\_c\_attn})
favors RN wherever the difference is statistically significant ($t=4,5,7$; at
$t=4$, $251.58\pm16.05$ under SR against $154.08$ under RN). The attention
output projection (\code{attn\_c\_proj}) favors SR at $t=4$--$6$, and the MLP
input projection (\code{mlp\_c\_fc}) favors SR at $t=5$--$6$; at $t\ge7$ neither
difference is statistically significant.

The largest SR advantage occurring at the longest reduction is consistent with
\cref{sec:analysis-mlp}: RN rounding errors can align across the $n$ terms of a
reduction, whereas SR errors have zero conditional mean, so RN drift can grow
faster with $n$ than SR fluctuation. The sweep does not isolate $n$, however.
Only one site has $n=3072$, and it also differs from the others in its input,
the GELU output, which is bounded below by approximately $-0.17$. The opposite
preferences of three sites with identical $n=768$ show that factors other than
$n$, such as the input distribution, the weights, and the downstream
operations, also determine the ordering.

\paragraph{SR's advantage across depth.}
The cumulative sweep in \cref{fig:cumulative} reduces \code{mlp\_c\_proj} in
progressively longer block prefixes, from block~$0$ to all six blocks. For every
tested precision, SR yields statistically significantly lower perplexity than RN
at every prefix, except $0$ and $0$--$1$ at $t=6$, where the differences are not
statistically significant.

The advantage of SR becomes more pronounced as more blocks use reduced
precision. At $t=4$, perplexities are initially close ($83.24$ under SR and
$85.13$ under RN for block~$0$), but extending the reduction to all six blocks
increases perplexity by a factor of $5.5$ under SR and $22.7$ under RN. At
$t=6$, the difference becomes statistically significant from prefix $0$--$2$
onward, reaching $66.24\pm0.63$ under SR versus $79.64$ under RN with all six
blocks reduced, as in \cref{fig:sublayer}. These results show that SR limits the
perplexity degradation from reducing the MLP down-projection across multiple
blocks. The growing SR--RN gap is consistent with RN introducing systematic drift
and SR limiting its accumulation across blocks.

This compounding effect illustrates a key structural difference between the
rounding rules. The deterministic drift introduced by RN in one block acts as a
correlated error that systematically shifts activations in subsequent blocks. In
contrast, the noise introduced by SR accumulates without compounding directional
drift across depth. While SR's unbiasedness does not survive non-linear
operations, empirically, the downstream propagated noise remains far less biased
for SR than for RN, validating its theoretical benefits across network depth.

\begin{figure}[t]
  \centering
  \safeincludegraphics{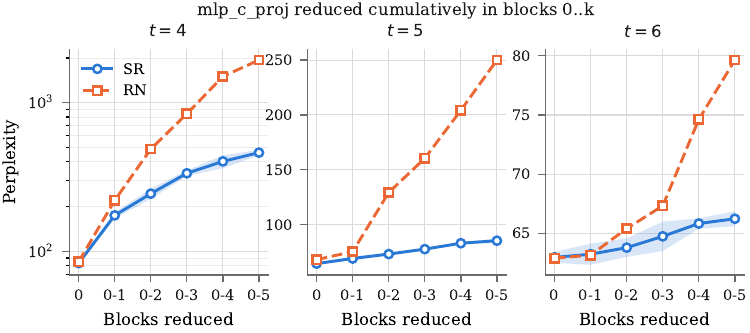}
  \caption{Cumulative propagation across block prefixes $0$ to $0\text{--}k$ for \code{mlp\_c\_proj} at $t \in \{4, 5, 6\}$. Bands show $\pm 1$ sample SD over five SR seeds.}
  \label{fig:cumulative}
\end{figure}

\subsection{Mixed rounding beats full RN at matched precision}
\label{sec:exp-composition}

The sweeps above reduce one site or one module at a time, and per-site effects
do not add (\cref{sec:exp-results}). We therefore test whether the per-site
preferences hold when five sites are reduced together. Each mixed configuration
assigns SR or RN per site and is compared with an all-RN configuration at
identical per-site precisions (\cref{tab:accuracy-matched-precision}).

The assignment follows the sublayer sweeps. SR is used at the two output
projections, where it is preferred in isolation: \code{attn\_c\_proj} at $t=6$
($62.10\pm0.42$ under SR against $62.91$ under RN) and \code{mlp\_c\_proj} at
$k\in\{6,7,8\}$ (statistically significant at $k=6,7$; not statistically
significant at $k=8$). RN is used at the two input projections at $t=8$, where
neither sublayer difference is statistically significant and
\code{attn\_c\_attn} favors RN at every lower precision with a statistically
significant difference. The head runs under RN at $t=11$, where head-only RN is
within $0.8\%$ of the reference ($63.00$). All remaining operations run at
$t=24$ under RN. The all-RN control replaces SR by RN at the two output
projections and changes nothing else.

\begin{table}[t]
  \centering
  \caption{Mixed rounding versus RN at identical per-site significand
    precisions on WikiText-2. Lower perplexity is better.}
  \label{tab:accuracy-matched-precision}
  \footnotesize
  \setlength{\tabcolsep}{5pt}
  \renewcommand{\arraystretch}{1.15}
  \begin{tabular}{@{}l c cc@{}}
    \toprule
    \multicolumn{3}{@{}l}{\textbf{(a) Precision and rounding configurations}}                                \\
    Site                                    & $t$ bits         & Mixed                             \\
    \midrule
    Attention input (\code{attn\_c\_attn})  & $8$              & RN                                \\
    Attention output (\code{attn\_c\_proj}) & $6$              & \textbf{SR}                       \\
    MLP input (\code{mlp\_c\_fc})           & $8$              & RN                                \\
    MLP output (\code{mlp\_c\_proj})        & $k \in \{6,7,8\} $          & \textbf{SR}            \\
    Language-model head (\code{lm\_head})   & $11$             & RN                                \\
    All other operations                    & $24$             & RN                                \\
    \bottomrule
  \end{tabular}\hfill
  \begin{minipage}[c]{0.42\linewidth}
    \footnotesize
    Configurations apply across all six blocks; $\mu\pm\sigma$ over five seeds.
    Negative $\Delta$ favors mixed rounding. CIs are $95\%$ Student-$t$ ($df=4$).
    Relative baseline column shows percentage increase over the full-precision reference ($62.51$) for Mixed (Full RN).
  \end{minipage}

  \medskip
  \begin{tabular}{@{}c c c c c c@{}}
    \toprule
    \multicolumn{6}{@{}l}{\textbf{(b) Perplexity gaps}}                                                                                                          \\
    \addlinespace[2pt]
    \shortstack{$k$}
        & \shortstack[r]{Mixed ($\mu\pm\sigma$) }
        & \shortstack[r]{Full RN}
        & \shortstack[r]{Gap $\Delta$ ($\%$)}
        & \shortstack{$95\%$ CI $\Delta$}
        & \shortstack[r]{Rel. baseline $\Delta$ (Full RN)}                                                                                                       \\
    \midrule
    $8$ & $64.63\pm0.68$                                   & $66.66$ & \SRbetter{$-2.03$ ($-3.04\%$)}   & $[-2.87,-1.18]$   & \RNbetter{$+3.38\%$ $(+6.64\%)$}   \\
    $7$ & $64.89\pm0.98$                                   & $68.89$ & \SRbetter{$-4.00$ ($-5.81\%$)}   & $[-5.22,-2.79]$   & \RNbetter{$+3.80\%$ $(+10.20\%)$}  \\
    $6$ & $69.01\pm2.07$                                   & $96.07$ & \SRbetter{$-27.06$ ($-28.17\%$)} & $[-29.63,-24.49]$ & \RNbetter{$+10.39\%$ $(+53.68\%)$} \\
    \bottomrule
  \end{tabular}
\end{table}

These results confirm that combining mixed rounding with mixed precision
consistently improves overall performance compared to using RN alone. At
identical per-site precisions, the mixed configuration gives lower perplexity
than all-RN from $3.04\%$ to $28.17\%$ for all three values of $k$. The gain
grows as the MLP down-projection loses precision, as in the sublayer sweep.
Across rows, the mixed configuration at $k=6$ ($69.01\pm2.07$) is not
distinguishable from all-RN at $k=7$ ($68.89$), and the mixed configuration at
$k=7$ ($64.89\pm0.98$) is below all-RN at $k=8$ ($66.66$), which lies outside
its $95\%$ interval. In this recipe, replacing RN by SR at the two output
projections compensates for removing one significand bit at the MLP
down-projection.

\subsection{Limitations}
\label{sec:exp-limitations}

PRISM emulates reduced significand precision while retaining the \code{binary32}
exponent range. These results concern significand and accumulator precision
under emulation; they do not model the exponent-range restrictions of narrow
hardware formats such as FP8 or NVFP4. Transfer to native hardware also depends
on its reduction order and arithmetic implementation. Non-target operations in
these sweeps use PRISM's untied RN at $t=24$, rather than native ties-to-even
(RNE); while exact tie events are rare in continuous floating-point activations, their frequency and impact have not been characterized.
Incorporating explicit exponent-range limitations remains future work.

The complete sweeps use one six-block model and one leading text prefix. The
measured site ordering and composition effects need replication across texts,
larger models, and broader tasks. The reduction bounds and loss decomposition
identify relevant quantities, but do not establish a universal rounding
assignment or quantitative portability of the observed gains.

%% file: conclusion.tex
\section{Discussion and conclusion}
\label{sec:conclusion}

The rounding rule, when treated as a site-specific design variable rather than a
network-wide constant, affects perplexity in opposite directions depending on
where it is applied.
Inside the feedforward network, SR outperforms RN by a wide margin at low
precision, reaching an $11.6\times$ perplexity reduction at $t=5$.
Cumulatively, reducing all six MLP down-projections to $t=6$ holds perplexity
degradation to $6.0\%$ over the full-precision baseline under SR, against $27\%$
under RN ($66.24$ against $79.64$ perplexity).
At the language-model head, the ranking reverses and RN yields substantially
lower excess loss.
While \citet{chmiel2023accurate} found that RN is generally preferred in the
forward pass, and contemporary low-precision formats such as NVIDIA's
NVFP4~\citep{nvidia2025nvfp4} employ SR selectively for gradient casts during
training, our results demonstrate that operation-level granularity is equally
consequential at inference: scoping SR specifically to the long reductions of the
MLP down-projections mitigates accumulated drift, while retaining RN at the
output softmax avoids the severe variance penalty that SR incurs there.
These site-specific preferences also enable mixed-rounding recipes: assigning SR
to the attention and MLP down-projections and RN to the remaining sites yields relative
perplexity reductions of $3.04\%$--$28.17\%$ over matched all-RN configurations
across the three tested configurations.

\Cref{sec:analysis} identifies which quantities determine the site-level
trade-off.
An exact decomposition separates the signed shift of the target logit from the
relative KL distortion of competing classes; a second-order expansion further
splits the distortion into a drift-curvature term $Q$ and a variance penalty
$V$.
At the MLP down-projection, RN's deterministic errors accumulate into a
non-uniform logit shift whose curvature cost exceeds SR's variance penalty,
which is small because the propagated SR noise is predominantly a uniform
vocabulary shift that the softmax ignores.
At the language-model head, SR's fluctuation $\xi$ is genuinely non-uniform and is
penalized directly by cross-entropy.
The decomposition establishes that the comparative advantage of
either rounding rule is governed not by raw error magnitude, but by whether the
resulting perturbations propagate as a uniform vocabulary shift or disperse as
non-uniform distortion across logits.

Establishing the broader generality of this trade-off requires evaluating the
decomposition across larger models, varied architectures, and diverse downstream
tasks.
Porting PRISM directly to GPUs (where an early prototype is in development)
and incorporating exponent-range checks into variable-precision frameworks
will be essential steps to scale this analysis to modern billion-parameter
architectures and complete narrow hardware formats.

\paragraph{Reproducibility.}
All evaluation pipelines, experiment scripts, and aggregated data are publicly available at \githubrepo.
\YC{TODO: Provide details on how to reproduce the experiments, including code,
  datasets, and hyperparameters. Github links, docker images and Zenodo tags.}

%% file: appendix.tex
\section{Proof of Exactness of the Rounding Decision}
\label{app:proof-exact-sign}

This appendix provides the complete proof of \cref{thm:exact-sign} from
\cref{sec:proof}, establishing that the floating-point evaluation of the
decision boundary in \cref{alg:vpsr} preserves the exact mathematical sign under
Hypotheses~1--3.

\paragraph{Sign normalization.}
Define the sign-normalized variables $\rho' = \rho \cdot s_{\rm dir}$, $\pi' = \pi \cdot s_{\rm dir}$, and $\tau' = \tau \cdot s_{\rm dir}$. Since $s_{\rm dir} \in \{-1, 1\}$, the theorem's conclusion is equivalent to
\[
(\rho' - \pi') + \tau' \ge 0 \iff \fl(\fl(\rho' - \pi') + \tau') \ge 0 .
\]
When $\rho \ne 0$, $s_{\rm dir} = \mathrm{sgn}(\rho)$, so $\rho' = |\rho| \ge 0$ and $\pi' = |\pi| \ge 0$ (since $\rho$ and $\pi$ share the same sign by Hypothesis~3). When $\rho = 0$, $s_{\rm dir} = \mathrm{sgn}(\tau)$, so $\rho' = 0$ and $\tau' = |\tau| \ge 0$.

\begin{lemma}[No collapse to zero]
\label[lemma]{lem:no-collapse}
Under Hypotheses~1--3 of \cref{thm:exact-sign}, if
$\rho' \ge \ulp_{\rm hw}$ and $(\rho' - \pi') + \tau' < 0$, then
$\fl(\fl(\rho' - \pi') + \tau') \ne 0$.
\end{lemma}

\begin{proof}
Assume for contradiction that the evaluated result collapses to zero: $\fl(\fl(\rho' - \pi') + \tau') = 0$.
An IEEE~754 addition rounds to $0$ if and only if the exact sum is $0$.
Thus, $\fl(\rho' - \pi') + \tau' = 0 \implies \fl(\rho' - \pi') = -\tau'$.
We now establish the conditions for Sterbenz's Lemma~\citep{sterbenz1974floating} ($\frac{\rho'}{2} \le \pi' \le 2\rho'$):

\begin{itemize}
\item \textbf{Lower bound} ($\pi' \ge \frac{\rho'}{2}$): 
Our starting assumption $(\rho' - \pi') + \tau' < 0$ rearranges directly to $\pi' > \rho' + \tau'$. By Hypotheses~1 and~2, $\tau' \ge -\ulp_{\rm hw}/2 \ge -\rho'/2$. Therefore, $\pi' > \rho' - \rho'/2 = \frac{\rho'}{2}$.

\item \textbf{Upper bound} ($\pi' \le 2\rho'$): 
The absolute rounding error of the subtraction is bounded by $\frac{1}{2}\ulp_{\rm hw}(-\tau')$. By Hypothesis~1, $|\tau'| \le \ulp_{\rm hw}/2$, and $\ulp_{\rm hw}(\cdot)$ is nondecreasing in magnitude, so $\frac{1}{2}\ulp_{\rm hw}(-\tau') \le \frac{1}{2}\ulp_{\rm hw}(\ulp_{\rm hw}/2) \le \ulp_{\rm hw}/2$. Thus, $|\fl(\rho' - \pi') - (\rho' - \pi')| = |-\tau' - (\rho' - \pi')| = |(\rho' - \pi') + \tau'| \le \ulp_{\rm hw}/2$. Applying the reverse triangle inequality yields $|\rho' - \pi'| - |\tau'| \le \ulp_{\rm hw}/2$. Since $|\tau'| \le \ulp_{\rm hw}/2$ and $\ulp_{\rm hw} \le \rho'$, we obtain $|\rho' - \pi'| \le \ulp_{\rm hw} \le \rho'$. This bound implies $\rho' - \pi' \ge -\rho'$, which gives $\pi' \le 2\rho'$.
\end{itemize}

Sterbenz's Lemma guarantees the subtraction is exact: $\fl(\rho' - \pi') = \rho' - \pi'$.
Substituting this exactness back into our derivation yields $(\rho' - \pi') + \tau' = 0$, which contradicts our starting assumption that $(\rho' - \pi') + \tau' < 0$. Therefore, the collapse to $0$ is impossible.
\end{proof}

\begin{proof}[Proof of \cref{thm:exact-sign}.]\par\nopagebreak
\noindent\textbf{Case $\rho = 0$.}\enspace When $\rho' = 0$, the subtraction $\fl(\rho' - \pi') = \fl(-\pi') = -\pi'$ is exact, since negation is exact in IEEE~754. The evaluation reduces to a single rounding $\fl(-\pi' + \tau')$.
For the forward direction ($\implies$), assume $(\rho' - \pi') + \tau' \ge 0$; then $\fl(-\pi' + \tau') \ge 0$ by monotonicity.
For the backward direction ($\impliedby$), assume $\fl(-\pi' + \tau') \ge 0$ but, for contradiction, $(\rho' - \pi') + \tau' < 0$. Monotonicity forces $\fl(-\pi' + \tau') \le 0$, hence the result is exactly~$0$. But an IEEE~754 addition of representable inputs rounds to~$0$ only when the exact sum is~$0$, a contradiction.

\smallskip
\noindent\textbf{Case $\rho \ne 0$.}\enspace Hypothesis~2 gives $\rho' = |\rho| \ge \ulp_{\rm hw}$.
For the forward direction ($\implies$), assume $(\rho' - \pi') + \tau' \ge 0$, i.e., $\rho' - \pi' \ge -\tau'$. By monotonicity, $\fl(\rho' - \pi') \ge \fl(-\tau') = -\tau'$, so $\fl(\rho' - \pi') + \tau' \ge 0$. Applying monotonicity once more, $\fl(\fl(\rho' - \pi') + \tau') \ge 0$.
For the backward direction ($\impliedby$), assume $\fl(\fl(\rho' - \pi') + \tau') \ge 0$ but, for contradiction, $(\rho' - \pi') + \tau' < 0$. By the same monotonicity argument with reversed inequalities, $\fl(\fl(\rho' - \pi') + \tau') \le 0$. Combined with the assumption, the floating-point result is exactly $0$. But \cref{lem:no-collapse} shows this collapse is impossible, a contradiction.\hfill\BlackBox
\let\endproof\relax
\end{proof}

%% file: appendix_drift.tex
\FloatBarrier
\section{Drift and variance measurements}
\label{app:drift-measurements}

\Cref{tab:decomposition} provides the numerical decomposition behind
\cref{sec:analysis-empirical-drift}. All estimates condition on the same
$N=1020$ scored positions.

\paragraph{Logit measurement protocol.} The evaluation uses the same DistilGPT-2
checkpoint and $N=1020$ scored WikiText-2 positions as the perplexity sweeps
(\cref{sec:exp-setup}), with full-precision reference loss $\mathcal L_0=4.1354$ nats
($\operatorname{PPL}_0=62.51$).
Unlike the RN-background sweeps in \cref{sec:exp-scoping}, instrumented operations
outside the lowered site use the matching rounding rule at $t=24$ (SR background for SR;
RN for RN). The SR perturbations therefore include full-precision background rounding;
an eight-seed SR control with no lowered site yields mean excess loss of
$(6.66\pm14.48)\times10^{-8}$ nats, verifying that the background contribution is
negligible. As noted in \cref{sec:analysis-empirical-drift}, the decomposition estimates condition
on this text using $S=32$ SR seeds and one deterministic RN run per precision ($t=4,\ldots,10$).

\paragraph{Finite-seed estimators.}
With $\hat b=S^{-1}\sum_s\Delta_s$,
$A=\tfrac{1}{2S}\sum_s\Var_p(\Delta_s)$, and
$B=\tfrac12\Var_p(\hat b)$, the unbiased quadratic estimators are
$\hat Q = \tfrac{SB-A}{S-1} = B - \tfrac{\hat V}{S}$ and
$\hat V = \tfrac{S(A-B)}{S-1}$.
We use $\hat D=g^\top\hat b$ and define $\hat R$ by subtracting these three
terms from the measured mean excess loss; agreement of their sum is therefore
not an independent validation. Standard errors are estimated by leaving out each seed in turn and
recomputing the estimator. Although $Q\ge0$, its
unbiased estimate can be negative: at the head and $t=8$, $\hat Q$ is
$-0.0017\pm0.0013$ at $S=8$ and $-0.0003\pm0.0004$ at $S=32$. Both estimates lie within two standard errors of zero.
We compute the variance penalty without forming the full covariance matrix via
$2V=\operatorname{tr}(H\Sigma)=\sum_jp_j\Sigma_{jj}-\Var_s(p^\top\Delta)$,
which requires only the coordinate variances and the variance of $p^\top\Delta$.

\newcommand{\hdneg}{\phantom{-}}
\newcommand{\hddig}{\phantom{0}}

\begin{table}[htbp]
  \centering\small
  \caption{Loss decomposition terms (\cref{eq:head-full-comparison}) in nats, averaged over $N=1020$ positions. SR reports mean $\pm$ 1 SE ($S=32$ seeds); RN is deterministic ($V_{\RN}=0$).}
  \label{tab:decomposition}
  {\footnotesize
    \setlength{\tabcolsep}{3pt}
    \begin{tabular}{ccrrrrrr}
      \toprule
                                                                                  &                         & \multicolumn{1}{c}{$t$} & \multicolumn{1}{c}{$D$}                              & \multicolumn{1}{c}{$Q$}
                                                                                  & \multicolumn{1}{c}{$V$} & \multicolumn{1}{c}{$R$} & \multicolumn{1}{c}{$\E[\delta\mathcal L]$}                                                                                                                                                                                                                                             \\
      \midrule
                                                                                  &                         & \cellcolor{gray!10}$4$  & \cellcolor{gray!10}$\hdneg0.049\hddig\pm0.043\hddig$ & \cellcolor{gray!10}${-}0.002\hddig\pm0.011\hddig$    & \cellcolor{gray!10}$23.204\hddig\pm0.047\hddig$      & \cellcolor{gray!10}${-}5.860\hddig\pm0.046\hddig$    & \cellcolor{gray!10}$\hdneg17.392\hddig\pm0.045\hddig$      \\
                                                                                  &                         & \cellcolor{gray!10}$5$  & \cellcolor{gray!10}${-}0.009\hddig\pm0.021\hddig$    & \cellcolor{gray!10}${-}0.001\hddig\pm0.006\hddig$    & \cellcolor{gray!10}$10.446\hddig\pm0.021\hddig$      & \cellcolor{gray!10}${-}1.141\hddig\pm0.022\hddig$    & \cellcolor{gray!10}$\hdneg\hddig9.295\hddig\pm0.025\hddig$ \\
                                                                                  &                         & \cellcolor{gray!10}$6$  & \cellcolor{gray!10}$\hdneg0.015\hddig\pm0.016\hddig$ & \cellcolor{gray!10}$\hdneg0.002\hddig\pm0.002\hddig$ & \cellcolor{gray!10}$\hddig4.606\hddig\pm0.009\hddig$ & \cellcolor{gray!10}${-}0.223\hddig\pm0.009\hddig$    & \cellcolor{gray!10}$\hdneg\hddig4.400\hddig\pm0.016\hddig$ \\
                                                                                  &                         & \cellcolor{gray!10}$7$  & \cellcolor{gray!10}$\hdneg0.016\hddig\pm0.010\hddig$ & \cellcolor{gray!10}$\hdneg0.001\hddig\pm0.001\hddig$ & \cellcolor{gray!10}$\hddig1.951\hddig\pm0.003\hddig$ & \cellcolor{gray!10}${-}0.062\hddig\pm0.003\hddig$    & \cellcolor{gray!10}$\hdneg\hddig1.907\hddig\pm0.011\hddig$ \\
                                                                                  &                         & \cellcolor{gray!10}$8$  & \cellcolor{gray!10}${-}0.008\hddig\pm0.008\hddig$    & \cellcolor{gray!10}${-}0.0003\pm0.0004$              & \cellcolor{gray!10}$\hddig0.766\hddig\pm0.001\hddig$ & \cellcolor{gray!10}${-}0.017\hddig\pm0.001\hddig$    & \cellcolor{gray!10}$\hdneg\hddig0.740\hddig\pm0.007\hddig$ \\
                                                                                  &                         & \cellcolor{gray!10}$9$  & \cellcolor{gray!10}$\hdneg0.003\hddig\pm0.003\hddig$ & \cellcolor{gray!10}$\hdneg0.000\hddig\pm0.000\hddig$ & \cellcolor{gray!10}$\hddig0.270\hddig\pm0.001\hddig$ & \cellcolor{gray!10}${-}0.003\hddig\pm0.000\hddig$    & \cellcolor{gray!10}$\hdneg\hddig0.271\hddig\pm0.003\hddig$ \\
                                                                                  & \multirow{-7}{*}{SR}    & \cellcolor{gray!10}$10$ & \cellcolor{gray!10}$\hdneg0.003\hddig\pm0.002\hddig$ & \cellcolor{gray!10}$\hdneg0.000\hddig\pm0.000\hddig$ & \cellcolor{gray!10}$\hddig0.084\hddig\pm0.000\hddig$ & \cellcolor{gray!10}$\hdneg0.000\hddig\pm0.000\hddig$ & \cellcolor{gray!10}$\hdneg\hddig0.087\hddig\pm0.002\hddig$ \\
      \cmidrule(l){2-8}
                                                                                  &                         & $4$                     & ${-}0.016\hddig\phantom{{}\pm0.0000}$                & $\hdneg3.166\hddig\phantom{{}\pm0.0000}$             & $\hddig0\phantom{.0000}\phantom{{}\pm0.0000}$        & $\hdneg1.974\hddig\phantom{{}\pm0.0000}$             & $\hdneg\hddig5.124\hddig\phantom{{}\pm0.0000}$             \\
                                                                                  &                         & $5$                     & $\hdneg0.104\hddig\phantom{{}\pm0.0000}$             & $\hdneg2.402\hddig\phantom{{}\pm0.0000}$             & $\hddig0\phantom{.0000}\phantom{{}\pm0.0000}$        & $\hdneg1.100\hddig\phantom{{}\pm0.0000}$             & $\hdneg\hddig3.605\hddig\phantom{{}\pm0.0000}$             \\
                                                                                  &                         & $6$                     & $\hdneg0.041\hddig\phantom{{}\pm0.0000}$             & $\hdneg1.576\hddig\phantom{{}\pm0.0000}$             & $\hddig0\phantom{.0000}\phantom{{}\pm0.0000}$        & $\hdneg0.500\hddig\phantom{{}\pm0.0000}$             & $\hdneg\hddig2.117\hddig\phantom{{}\pm0.0000}$             \\
                                                                                  &                         & $7$                     & $\hdneg0.030\hddig\phantom{{}\pm0.0000}$             & $\hdneg0.830\hddig\phantom{{}\pm0.0000}$             & $\hddig0\phantom{.0000}\phantom{{}\pm0.0000}$        & $\hdneg0.168\hddig\phantom{{}\pm0.0000}$             & $\hdneg\hddig1.028\hddig\phantom{{}\pm0.0000}$             \\
                                                                                  &                         & $8$                     & $\hdneg0.019\hddig\phantom{{}\pm0.0000}$             & $\hdneg0.346\hddig\phantom{{}\pm0.0000}$             & $\hddig0\phantom{.0000}\phantom{{}\pm0.0000}$        & $\hdneg0.038\hddig\phantom{{}\pm0.0000}$             & $\hdneg\hddig0.404\hddig\phantom{{}\pm0.0000}$             \\
                                                                                  &                         & $9$                     & ${-}0.012\hddig\phantom{{}\pm0.0000}$                & $\hdneg0.124\hddig\phantom{{}\pm0.0000}$             & $\hddig0\phantom{.0000}\phantom{{}\pm0.0000}$        & $\hdneg0.007\hddig\phantom{{}\pm0.0000}$             & $\hdneg\hddig0.120\hddig\phantom{{}\pm0.0000}$             \\
      \multirow{-14}{*}{\rotatebox[origin=c]{90}{\code{lm\_head} ($n=768$)}}      & \multirow{-7}{*}{RN}    & $10$                    & ${-}0.014\hddig\phantom{{}\pm0.0000}$                & $\hdneg0.039\hddig\phantom{{}\pm0.0000}$             & $\hddig0\phantom{.0000}\phantom{{}\pm0.0000}$        & $\hdneg0.000\hddig\phantom{{}\pm0.0000}$             & $\hdneg\hddig0.025\hddig\phantom{{}\pm0.0000}$             \\
      \midrule
                                                                                  &                         & \cellcolor{gray!10}$4$  & \cellcolor{gray!10}${-}0.2553\pm0.0064$              & \cellcolor{gray!10}$\hdneg0.9690\pm0.0042$           & \cellcolor{gray!10}$\hddig0.4566\pm0.0020$           & \cellcolor{gray!10}$\hdneg0.8280\pm0.0062$           & \cellcolor{gray!10}$\hdneg\hddig1.9983\pm0.0085$           \\
                                                                                  &                         & \cellcolor{gray!10}$5$  & \cellcolor{gray!10}${-}0.1300\pm0.0047$              & \cellcolor{gray!10}$\hdneg0.1645\pm0.0022$           & \cellcolor{gray!10}$\hddig0.2245\pm0.0012$           & \cellcolor{gray!10}$\hdneg0.0680\pm0.0024$           & \cellcolor{gray!10}$\hdneg\hddig0.3270\pm0.0054$           \\
                                                                                  &                         & \cellcolor{gray!10}$6$  & \cellcolor{gray!10}${-}0.0532\pm0.0026$              & \cellcolor{gray!10}$\hdneg0.0165\pm0.0005$           & \cellcolor{gray!10}$\hddig0.0828\pm0.0005$           & \cellcolor{gray!10}$\hdneg0.0040\pm0.0004$           & \cellcolor{gray!10}$\hdneg\hddig0.0501\pm0.0027$           \\
                                                                                  &                         & \cellcolor{gray!10}$7$  & \cellcolor{gray!10}${-}0.0169\pm0.0019$              & \cellcolor{gray!10}$\hdneg0.0014\pm0.0001$           & \cellcolor{gray!10}$\hddig0.0252\pm0.0002$           & \cellcolor{gray!10}$\hdneg0.0006\pm0.0001$           & \cellcolor{gray!10}$\hdneg\hddig0.0104\pm0.0018$           \\
                                                                                  &                         & \cellcolor{gray!10}$8$  & \cellcolor{gray!10}${-}0.0043\pm0.0008$              & \cellcolor{gray!10}$\hdneg0.0001\pm0.0000$           & \cellcolor{gray!10}$\hddig0.0067\pm0.0001$           & \cellcolor{gray!10}$\hdneg0.0001\pm0.0000$           & \cellcolor{gray!10}$\hdneg\hddig0.0026\pm0.0007$           \\
                                                                                  &                         & \cellcolor{gray!10}$9$  & \cellcolor{gray!10}${-}0.0013\pm0.0004$              & \cellcolor{gray!10}$\hdneg0.0000\pm0.0000$           & \cellcolor{gray!10}$\hddig0.0017\pm0.0000$           & \cellcolor{gray!10}$\hdneg0.0000\pm0.0000$           & \cellcolor{gray!10}$\hdneg\hddig0.0005\pm0.0004$           \\
                                                                                  & \multirow{-7}{*}{SR}    & \cellcolor{gray!10}$10$ & \cellcolor{gray!10}${-}0.0004\pm0.0002$              & \cellcolor{gray!10}$\hdneg0.0000\pm0.0000$           & \cellcolor{gray!10}$\hddig0.0004\pm0.0000$           & \cellcolor{gray!10}$\hdneg0.0000\pm0.0000$           & \cellcolor{gray!10}$\hdneg\hddig0.0001\pm0.0002$           \\
      \cmidrule(l){2-8}
                                                                                  &                         & $4$                     & ${-}0.1676\phantom{{}\pm0.0000}$                     & $\hdneg2.4204\phantom{{}\pm0.0000}$                  & $\hddig0\phantom{.0000}\phantom{{}\pm0.0000}$        & $\hdneg1.1792\phantom{{}\pm0.0000}$                  & $\hdneg\hddig3.4321\phantom{{}\pm0.0000}$                  \\
                                                                                  &                         & $5$                     & ${-}0.1250\phantom{{}\pm0.0000}$                     & $\hdneg1.0504\phantom{{}\pm0.0000}$                  & $\hddig0\phantom{.0000}\phantom{{}\pm0.0000}$        & $\hdneg0.4607\phantom{{}\pm0.0000}$                  & $\hdneg\hddig1.3861\phantom{{}\pm0.0000}$                  \\
                                                                                  &                         & $6$                     & ${-}0.1148\phantom{{}\pm0.0000}$                     & $\hdneg0.2867\phantom{{}\pm0.0000}$                  & $\hddig0\phantom{.0000}\phantom{{}\pm0.0000}$        & $\hdneg0.0702\phantom{{}\pm0.0000}$                  & $\hdneg\hddig0.2421\phantom{{}\pm0.0000}$                  \\
                                                                                  &                         & $7$                     & ${-}0.0288\phantom{{}\pm0.0000}$                     & $\hdneg0.0686\phantom{{}\pm0.0000}$                  & $\hddig0\phantom{.0000}\phantom{{}\pm0.0000}$        & $\hdneg0.0055\phantom{{}\pm0.0000}$                  & $\hdneg\hddig0.0453\phantom{{}\pm0.0000}$                  \\
                                                                                  &                         & $8$                     & ${-}0.0101\phantom{{}\pm0.0000}$                     & $\hdneg0.0166\phantom{{}\pm0.0000}$                  & $\hddig0\phantom{.0000}\phantom{{}\pm0.0000}$        & $\hdneg0.0006\phantom{{}\pm0.0000}$                  & $\hdneg\hddig0.0071\phantom{{}\pm0.0000}$                  \\
                                                                                  &                         & $9$                     & ${-}0.0043\phantom{{}\pm0.0000}$                     & $\hdneg0.0040\phantom{{}\pm0.0000}$                  & $\hddig0\phantom{.0000}\phantom{{}\pm0.0000}$        & $\hdneg0.0000\phantom{{}\pm0.0000}$                  & $\hddig{-}0.0002\phantom{{}\pm0.0000}$                     \\
      \multirow{-14}{*}{\rotatebox[origin=c]{90}{\code{mlp\_c\_proj} ($n=3072$)}} & \multirow{-7}{*}{RN}    & $10$                    & ${-}0.0020\phantom{{}\pm0.0000}$                     & $\hdneg0.0010\phantom{{}\pm0.0000}$                  & $\hddig0\phantom{.0000}\phantom{{}\pm0.0000}$        & $\hdneg0.0000\phantom{{}\pm0.0000}$                  & $\hddig{-}0.0009\phantom{{}\pm0.0000}$                     \\
      \bottomrule
    \end{tabular}}
\end{table}
\FloatBarrier

%% file: main.bbl
\begin{thebibliography}{48}
\providecommand{\natexlab}[1]{#1}
\providecommand{\url}[1]{\texttt{#1}}
\expandafter\ifx\csname urlstyle\endcsname\relax
  \providecommand{\doi}[1]{doi: #1}\else
  \providecommand{\doi}{doi: \begingroup \urlstyle{rm}\Url}\fi

\bibitem[Abecassis et~al.(2025)Abecassis, Agrusa, Ahn, Alben,
  et~al.]{nvidia2025nvfp4}
F.~Abecassis, A.~Agrusa, D.~Ahn, J.~Alben, et~al.
\newblock Pretraining large language models with {NVFP4}.
\newblock \emph{arXiv preprint arXiv:2509.25149}, 2025.

\bibitem[{Amazon Web Services}(2023)]{aws2023neuroncore}
{Amazon Web Services}.
\newblock {NeuronCore-v2} architecture.
\newblock Neuron Documentation,
  \url{https://awsdocs-neuron.readthedocs-hosted.com}, 2023.

\bibitem[Ashkboos et~al.(2024)Ashkboos, Mohtashami, Croci, Li, Cameron, Jaggi,
  Alistarh, Hoefler, and Hensman]{ashkboos2024quarot}
S.~Ashkboos, A.~Mohtashami, M.~L. Croci, B.~Li, P.~Cameron, M.~Jaggi,
  D.~Alistarh, T.~Hoefler, and J.~Hensman.
\newblock {QuaRot}: Outlier-free 4-bit inference in rotated {LLMs}.
\newblock In \emph{Advances in Neural Information Processing Systems},
  volume~37, pages 100213--100240, 2024.
\newblock \doi{10.52202/079017-3180}.

\bibitem[Ben~Ali et~al.(2024)Ben~Ali, Filip, and
  Sentieys]{benali2024stochastic}
S.~Ben~Ali, S.-I. Filip, and O.~Sentieys.
\newblock A stochastic rounding-enabled low-precision floating-point {MAC} for
  {DNN} training.
\newblock In \emph{2024 Design, Automation \& Test in Europe Conference \&
  Exhibition (DATE)}, pages 1--6. IEEE, 2024.
\newblock \doi{10.23919/DATE58400.2024.10546735}.

\bibitem[Beuzeville et~al.(2026)Beuzeville, Buttari, Gratton, and
  Mary]{beuzeville2025deterministic}
T.~Beuzeville, A.~Buttari, S.~Gratton, and T.~Mary.
\newblock Deterministic and probabilistic rounding error analysis of neural
  networks in floating-point arithmetic.
\newblock \emph{IMA Journal of Numerical Analysis}, 2026.
\newblock \doi{10.1093/imanum/draf130}.

\bibitem[Blanchard et~al.(2021)Blanchard, Higham, and
  Higham]{blanchard2021accurately}
P.~Blanchard, D.~J. Higham, and N.~J. Higham.
\newblock Accurately computing the log-sum-exp and softmax functions.
\newblock \emph{IMA Journal of Numerical Analysis}, 41\penalty0 (4):\penalty0
  2311--2330, 2021.
\newblock \doi{10.1093/imanum/draa038}.

\bibitem[Boldo and Muller(2011)]{boldo2011exact}
S.~Boldo and J.-M. Muller.
\newblock Exact and approximated error of the {FMA}.
\newblock \emph{IEEE Transactions on Computers}, 60\penalty0 (2):\penalty0
  157--164, 2011.
\newblock \doi{10.1109/TC.2010.139}.

\bibitem[Chatelain(2026)]{prism}
Y.~Chatelain.
\newblock {PRISM}: Probabilistic rounding with instruction set management.
\newblock \url{https://github.com/verificarlo/prism}, 2026.

\bibitem[Chmiel et~al.(2023)Chmiel, Banner, Hoffer, Ben-Yaacov, and
  Soudry]{chmiel2023accurate}
B.~Chmiel, R.~Banner, E.~Hoffer, H.~Ben-Yaacov, and D.~Soudry.
\newblock Accurate neural training with 4-bit matrix multiplications at
  standard formats.
\newblock In \emph{International Conference on Learning Representations}, 2023.

\bibitem[Connolly et~al.(2021)Connolly, Higham, and
  Mary]{connolly2021stochastic}
M.~P. Connolly, N.~J. Higham, and T.~Mary.
\newblock Stochastic rounding and its probabilistic backward error analysis.
\newblock \emph{SIAM Journal on Scientific Computing}, 43\penalty0
  (1):\penalty0 A566--A585, 2021.
\newblock \doi{10.1137/20M1334796}.

\bibitem[Croci et~al.(2022)Croci, Fasi, Higham, Mary, and
  Mikaitis]{croci2022stochastic}
M.~Croci, M.~Fasi, N.~J. Higham, T.~Mary, and M.~Mikaitis.
\newblock Stochastic rounding: implementation, error analysis and applications.
\newblock \emph{Royal Society Open Science}, 9\penalty0 (3):\penalty0 211631,
  2022.
\newblock \doi{10.1098/rsos.211631}.

\bibitem[de~Oliveira~Castro et~al.(2025)de~Oliveira~Castro, El~Arar, Petit, and
  Sohier]{de2025error}
P.~de~Oliveira~Castro, E.-M. El~Arar, E.~Petit, and D.~Sohier.
\newblock Error analysis of sum-product algorithms under stochastic rounding.
\newblock \emph{SIAM Journal on Scientific Computing}, 47\penalty0
  (6):\penalty0 B1481--B1502, 2025.
\newblock \doi{10.1137/24M1710966}.

\bibitem[Dekker(1971)]{dekker1971float}
T.~J. Dekker.
\newblock A floating-point technique for extending the available precision.
\newblock \emph{Numerische Mathematik}, 18\penalty0 (3):\penalty0 224--242,
  1971.
\newblock \doi{10.1007/BF01397083}.

\bibitem[Denis et~al.(2016)Denis, de~Oliveira~Castro, and
  Petit]{denis2016verificarlo}
C.~Denis, P.~de~Oliveira~Castro, and E.~Petit.
\newblock Verificarlo: checking floating point accuracy through {Monte Carlo}
  arithmetic.
\newblock In \emph{2016 IEEE 23rd Symposium on Computer Arithmetic (ARITH)},
  pages 55--62, 2016.
\newblock \doi{10.1109/ARITH.2016.31}.

\bibitem[Dettmers et~al.(2022)Dettmers, Lewis, Belkada, and
  Zettlemoyer]{3600270.3602468}
T.~Dettmers, M.~Lewis, Y.~Belkada, and L.~Zettlemoyer.
\newblock {LLM.int8()}: 8-bit matrix multiplication for transformers at scale.
\newblock In \emph{Proceedings of the 36th International Conference on Neural
  Information Processing Systems}, NIPS '22, pages 30318--30332, Red Hook, NY,
  USA, 2022. Curran Associates Inc.
\newblock ISBN 9781713871088.
\newblock \doi{10.52202/068431-2198}.

\bibitem[El~Arar et~al.(2023)El~Arar, Sohier, de~Oliveira~Castro, and
  Petit]{elarar2022error}
E.-M. El~Arar, D.~Sohier, P.~de~Oliveira~Castro, and E.~Petit.
\newblock Stochastic rounding variance and probabilistic bounds: A new
  approach.
\newblock \emph{SIAM Journal on Scientific Computing}, 45\penalty0
  (5):\penalty0 C255--C275, 2023.
\newblock \doi{10.1137/22M1510819}.

\bibitem[El~Arar et~al.(2024)El~Arar, Sohier, de~Oliveira~Castro, and
  Petit]{elarar2024bounds}
E.-M. El~Arar, D.~Sohier, P.~de~Oliveira~Castro, and E.~Petit.
\newblock Bounds on non-linear errors for variance computation with stochastic
  rounding.
\newblock \emph{SIAM Journal on Scientific Computing}, 46\penalty0
  (5):\penalty0 B579--B599, 2024.
\newblock \doi{10.1137/23M1563001}.

\bibitem[Fasi and Mikaitis(2021)]{fasi2021algorithms}
M.~Fasi and M.~Mikaitis.
\newblock Algorithms for stochastically rounded elementary arithmetic
  operations in {IEEE} 754 floating-point arithmetic.
\newblock \emph{IEEE Transactions on Emerging Topics in Computing}, 9\penalty0
  (3):\penalty0 1451--1466, 2021.
\newblock \doi{10.1109/TETC.2021.3069165}.

\bibitem[F{\'e}votte and Lathuili{\`e}re(2016)]{fevotte2016verrou}
F.~F{\'e}votte and B.~Lathuili{\`e}re.
\newblock {VERROU}: Assessing floating-point accuracy without recompiling.
\newblock Technical report, EDF R\&D, 2016.
\newblock HAL: hal-01383417.

\bibitem[Forsythe(1959)]{forsythe1959rounding}
G.~E. Forsythe.
\newblock Reprint of a note on rounding-off errors.
\newblock \emph{SIAM Review}, 1\penalty0 (1):\penalty0 66--67, 1959.
\newblock \doi{10.1137/1001011}.

\bibitem[Gong et~al.(2024)Gong, Liu, Wang, Cai, Zhao, and Yan]{gong2024what}
Z.~Gong, J.~Liu, J.~Wang, X.~Cai, D.~Zhao, and R.~Yan.
\newblock What makes quantization for large language model hard? an empirical
  study from the lens of perturbation.
\newblock \emph{Proceedings of the AAAI Conference on Artificial Intelligence},
  38\penalty0 (16):\penalty0 18082--18089, 2024.
\newblock \doi{10.1609/aaai.v38i16.29765}.

\bibitem[Gonz{\'a}lez-Pepe et~al.(2026)Gonz{\'a}lez-Pepe, Akhaddar, Glatard,
  and Chatelain]{gonzalezpepe2026fuzzy}
I.~Gonz{\'a}lez-Pepe, H.~Akhaddar, T.~Glatard, and Y.~Chatelain.
\newblock Fuzzy {PyTorch}: Rapid numerical variability evaluation for deep
  learning models.
\newblock \emph{Transactions on Machine Learning Research}, 2026.
\newblock arXiv:2605.25991.

\bibitem[{Google}(2024)]{highway}
{Google}.
\newblock Highway: Efficient and portable {SIMD}/{vector} {C++} library.
\newblock \url{https://github.com/google/highway}, 2024.

\bibitem[Graillat and Muller(2023)]{graillat2023emulation}
S.~Graillat and J.-M. Muller.
\newblock Emulation of the {FMA} in rounded-to-nearest floating-point
  arithmetic.
\newblock \emph{Journal of Computational and Applied Mathematics},
  433:\penalty0 115304, 2023.
\newblock \doi{10.1016/j.cam.2023.115304}.

\bibitem[Graillat et~al.(2006)Graillat, Langlois, and
  Louvet]{graillat2006improving}
S.~Graillat, P.~Langlois, and N.~Louvet.
\newblock Improving the compensated {Horner} scheme with a fused multiply and
  add.
\newblock In \emph{Proceedings of the 2006 ACM Symposium on Applied Computing},
  pages 1323--1327, 2006.
\newblock \doi{10.1145/1141277.1141585}.

\bibitem[{Graphcore}(2020)]{graphcore2020aifloat}
{Graphcore}.
\newblock {AI-Float}: Mixed-precision arithmetic for {AI}, a hardware
  perspective.
\newblock Technical white paper, \url{https://docs.graphcore.ai}, 2020.

\bibitem[Gupta et~al.(2015)Gupta, Agrawal, Gopalakrishnan, and
  Narayanan]{gupta2015deep}
S.~Gupta, A.~Agrawal, K.~Gopalakrishnan, and P.~Narayanan.
\newblock Deep learning with limited numerical precision.
\newblock In \emph{Proceedings of the 32nd International Conference on Machine
  Learning}, volume~37 of \emph{Proceedings of Machine Learning Research},
  pages 1737--1746. PMLR, 2015.

\bibitem[Hendrycks and Gimpel(2016)]{hendrycks2016gelu}
D.~Hendrycks and K.~Gimpel.
\newblock Gaussian error linear units ({GELUs}).
\newblock \emph{arXiv preprint arXiv:1606.08415}, 2016.

\bibitem[Higham(2002)]{higham2002accuracy}
N.~J. Higham.
\newblock \emph{Accuracy and Stability of Numerical Algorithms}.
\newblock SIAM, Philadelphia, PA, 2nd edition, 2002.
\newblock \doi{10.1137/1.9780898718027}.

\bibitem[Higham and Mary(2019)]{higham2019new}
N.~J. Higham and T.~Mary.
\newblock A new approach to probabilistic rounding error analysis.
\newblock \emph{SIAM Journal on Scientific Computing}, 41\penalty0
  (5):\penalty0 A2815--A2835, 2019.
\newblock \doi{10.1137/18M1226312}.

\bibitem[{Hugging Face}(2020)]{distilgpt2modelcard}
{Hugging Face}.
\newblock {DistilGPT2} model card.
\newblock \url{https://huggingface.co/distilbert/distilgpt2}, 2020.

\bibitem[{IEEE}(2019)]{ieee754-2019}
{IEEE}.
\newblock {IEEE} standard for floating-point arithmetic.
\newblock IEEE Std 754-2019, 2019.

\bibitem[Ipsen and Zhou(2020)]{ipsen2020probabilistic}
I.~C.~F. Ipsen and H.~Zhou.
\newblock Probabilistic error analysis for inner products.
\newblock \emph{SIAM Journal on Matrix Analysis and Applications}, 41\penalty0
  (4):\penalty0 1726--1741, 2020.
\newblock \doi{10.1137/19M1270434}.

\bibitem[Kim et~al.(2024)Kim, Hooper, Gholami, Dong, Li, Shen, Mahoney, and
  Keutzer]{kim2024squeezellm}
S.~Kim, C.~Hooper, A.~Gholami, Z.~Dong, X.~Li, S.~Shen, M.~W. Mahoney, and
  K.~Keutzer.
\newblock {SqueezeLLM}: Dense-and-sparse quantization.
\newblock In \emph{Proceedings of the 41st International Conference on Machine
  Learning}, volume 235 of \emph{Proceedings of Machine Learning Research},
  pages 23901--23923. PMLR, 2024.

\bibitem[Lin et~al.(2024)Lin, Tang, Tang, Yang, Chen, Wang, Xiao, Dang, Gan,
  and Han]{lin2024awq}
J.~Lin, J.~Tang, H.~Tang, S.~Yang, W.-M. Chen, W.-C. Wang, G.~Xiao, X.~Dang,
  C.~Gan, and S.~Han.
\newblock {AWQ}: Activation-aware weight quantization for on-device {LLM}
  compression and acceleration.
\newblock In \emph{Proceedings of Machine Learning and Systems}, volume~6,
  pages 87--100, 2024.

\bibitem[Martens(2020)]{martens2020new}
J.~Martens.
\newblock New insights and perspectives on the natural gradient method.
\newblock \emph{Journal of Machine Learning Research}, 21\penalty0
  (146):\penalty0 1--76, 2020.
\newblock URL \url{https://www.jmlr.org/papers/v21/17-678.html}.

\bibitem[Merity et~al.(2017)Merity, Xiong, Bradbury, and
  Socher]{merity2017pointer}
S.~Merity, C.~Xiong, J.~Bradbury, and R.~Socher.
\newblock Pointer sentinel mixture models.
\newblock In \emph{International Conference on Learning Representations}, 2017.

\bibitem[Muller et~al.(2018)Muller, Brunie, de~Dinechin, Jeannerod, Joldes,
  Lef{\`e}vre, Melquiond, Revol, and Torres]{muller2018handbook}
J.-M. Muller, N.~Brunie, F.~de~Dinechin, C.-P. Jeannerod, M.~Joldes,
  V.~Lef{\`e}vre, G.~Melquiond, N.~Revol, and S.~Torres.
\newblock \emph{Handbook of Floating-Point Arithmetic}.
\newblock Birkh{\"a}user, 2nd edition, 2018.
\newblock \doi{10.1007/978-3-319-76526-6}.

\bibitem[Nagel et~al.(2020)Nagel, Amjad, van Baalen, Louizos, and
  Blankevoort]{nagel2020adaround}
M.~Nagel, R.~A. Amjad, M.~van Baalen, C.~Louizos, and T.~Blankevoort.
\newblock Up or down? {A}daptive rounding for post-training quantization.
\newblock In \emph{Proceedings of the 37th International Conference on Machine
  Learning}, volume 119 of \emph{Proceedings of Machine Learning Research},
  pages 7197--7206. PMLR, 2020.

\bibitem[{NVIDIA}(2025)]{nvidia2025ptx}
{NVIDIA}.
\newblock Parallel thread execution {ISA}: The {\texttt{cvt.rs}}
  stochastic-rounding instruction.
\newblock CUDA Documentation, \url{https://docs.nvidia.com/cuda}, 2025.

\bibitem[Parker(1997)]{parker1997monte}
D.~S. Parker.
\newblock Monte carlo arithmetic: exploiting randomness in floating-point
  arithmetic.
\newblock Technical Report CSD-970002, University of California, Los Angeles,
  1997.

\bibitem[Radford et~al.(2019)Radford, Wu, Child, Luan, Amodei, and
  Sutskever]{radford2019language}
A.~Radford, J.~Wu, R.~Child, D.~Luan, D.~Amodei, and I.~Sutskever.
\newblock Language models are unsupervised multitask learners.
\newblock Technical report, OpenAI, 2019.

\bibitem[Sterbenz(1974)]{sterbenz1974floating}
P.~H. Sterbenz.
\newblock \emph{Floating-Point Computation}.
\newblock Prentice-Hall, Englewood Cliffs, NJ, 1974.

\bibitem[Vaswani et~al.(2017)Vaswani, Shazeer, Parmar, Uszkoreit, Jones, Gomez,
  Kaiser, and Polosukhin]{vaswani2017attention}
A.~Vaswani, N.~Shazeer, N.~Parmar, J.~Uszkoreit, L.~Jones, A.~N. Gomez,
  L.~Kaiser, and I.~Polosukhin.
\newblock Attention is all you need.
\newblock In \emph{Advances in Neural Information Processing Systems},
  volume~30, pages 5998--6008, 2017.

\bibitem[Wolf et~al.(2020)Wolf, Debut, Sanh, Chaumond, Delangue, Moi, Cistac,
  Rault, Louf, Funtowicz, et~al.]{wolf2020transformers}
T.~Wolf, L.~Debut, V.~Sanh, J.~Chaumond, C.~Delangue, A.~Moi, P.~Cistac,
  T.~Rault, R.~Louf, M.~Funtowicz, et~al.
\newblock Transformers: State-of-the-art natural language processing.
\newblock In \emph{Proceedings of the 2020 Conference on Empirical Methods in
  Natural Language Processing: System Demonstrations}, pages 38--45, 2020.
\newblock \doi{10.18653/v1/2020.emnlp-demos.6}.

\bibitem[Xiao et~al.(2023)Xiao, Lin, Seznec, Wu, Demouth, and
  Han]{xiao2023smoothquant}
G.~Xiao, J.~Lin, M.~Seznec, H.~Wu, J.~Demouth, and S.~Han.
\newblock {SmoothQuant}: Accurate and efficient post-training quantization for
  large language models.
\newblock In \emph{Proceedings of the 40th International Conference on Machine
  Learning}, volume 202 of \emph{Proceedings of Machine Learning Research},
  pages 38087--38099. PMLR, 2023.

\bibitem[Xiong et~al.(2020)Xiong, Yang, He, Zheng, Zheng, Xing, Zhang, Lan,
  Wang, and Liu]{xiong2020layer}
R.~Xiong, Y.~Yang, D.~He, K.~Zheng, S.~Zheng, C.~Xing, H.~Zhang, Y.~Lan,
  L.~Wang, and T.-Y. Liu.
\newblock On layer normalization in the transformer architecture.
\newblock In \emph{Proceedings of the 37th International Conference on Machine
  Learning}, volume 119 of \emph{Proceedings of Machine Learning Research},
  pages 10524--10533. PMLR, 2020.

\bibitem[Zhang et~al.(2019)Zhang, Lin, Yang, and De~Sa]{zhang2019qpytorch}
T.~Zhang, Z.~Lin, G.~Yang, and C.~De~Sa.
\newblock {QPyTorch}: A low-precision arithmetic simulation framework.
\newblock In \emph{2019 Fifth Workshop on Energy Efficient Machine Learning and
  Cognitive Computing - NeurIPS Edition (EMC2-NIPS)}, pages 10--13, 2019.
\newblock \doi{10.1109/EMC2-NIPS53020.2019.00010}.

\end{thebibliography}
